\documentclass[12pt]{article}
\usepackage{algorithm,rotating}
\usepackage[authoryear,round]{natbib}
\RequirePackage{hyperref}
\usepackage{amsmath,latexsym,amssymb,graphicx,amsthm}

\newtheorem{proposition}{Proposition}
\newtheorem{theorem}{Theorem}

\usepackage{setspace}
\DeclareMathOperator{\W}{\mathbf{W}}
\DeclareMathOperator{\V}{\mathbf{V}}
\DeclareMathOperator{\U}{\mathbf{U}}

\DeclareMathOperator{\uvec}{\mathbf{u}}
\DeclareMathOperator{\vvec}{\mathbf{v}}

\DeclareMathOperator{\wvec}{\mathbf{w}}
\DeclareMathOperator{\D}{\mathbf{D}}
\DeclareMathOperator{\Omeg}{\mathbf{\Omega}}
\DeclareMathOperator{\Pmat}{\mathbf{P}}

\DeclareMathOperator{\Q}{\mathbf{Q}}

\DeclareMathOperator{\Y}{\mathbf{Y}}

\DeclareMathOperator{\Hmat}{\mathbf{H}}

\DeclareMathOperator{\Su}{\mathbf{S}_{\mathbf{u}}}
\DeclareMathOperator{\Sv}{\mathbf{S}_{\mathbf{v}}}
\DeclareMathOperator{\Sw}{\mathbf{S}_{\mathbf{w}}}
\DeclareMathOperator{\X}{\mathbf{X}}
\DeclareMathOperator{\tenX}{\boldsymbol{\mathcal{X}}}
\DeclareMathOperator{\tenD}{\boldsymbol{\mathcal{D}}}
\DeclareMathOperator{\tenE}{\boldsymbol{\mathcal{E}}}
\DeclareMathOperator{\x}{\mathbf{x}}
\DeclareMathOperator{\y}{\mathbf{y}}

\DeclareMathOperator{\Smat}{\mathbf{S}}

\DeclareMathOperator*{\minimize}{\mathrm{minimize}}
\DeclareMathOperator*{\argmin}{\mathrm{argmin}}
\DeclareMathOperator*{\maximize}{\mathrm{maximize}}

\DeclareMathOperator*{\lamv}{\lambda_{\mathbf{v}}}
\DeclareMathOperator*{\lamu}{\lambda_{\mathbf{u}}}
\DeclareMathOperator*{\lamw}{\lambda_{\mathbf{w}}}

\begin{document}

\title{ \bf \Large Regularized Tensor Factorizations and Higher-Order
  Principal Components Analysis}

\author{Genevera I. Allen\footnotemark }
\footnotetext{Department of Pediatrics-Neurology, Baylor College of
  Medicine, Jan and Dan Duncan Neurological Research Institute, Texas
  Children's Hospital, \& Department of Statistics, Rice University, MS 138,
6100 Main St., Houston, TX 77005 (email:  gallen@rice.edu)}

\date{}

\maketitle

\begin{abstract}
High-dimensional tensors or multi-way data are becoming prevalent in
areas such as biomedical imaging, chemometrics, networking and
bibliometrics.  Traditional approaches to finding lower
dimensional representations of tensor data include flattening the data
and applying matrix factorizations such as principal components
analysis (PCA) or employing tensor decompositions
such as the CANDECOMP / PARAFAC (CP) and Tucker decompositions.  The
former can lose important structure in the data, while the latter
Higher-Order PCA (HOPCA) methods can
be problematic in high-dimensions with many irrelevant features.
We introduce frameworks for sparse tensor
factorizations or Sparse HOPCA based on heuristic algorithmic
approaches and by solving penalized optimization problems related to
the CP decomposition.  Extensions of these approaches lead to
methods for general regularized tensor factorizations, multi-way
Functional HOPCA and generalizations of HOPCA
for structured data.   We
illustrate the utility of our methods for dimension reduction, feature
selection, and signal recovery on 
simulated data and multi-dimensional microarrays and functional MRIs.  
\end{abstract}

{\bf Keywords:} tensor decompositions, principal components analysis,
  sparse PCA, functional PCA, generalized PCA, multi-way data

\section{Introduction}

Recently high-dimensional tensor data has become prevalent in areas of
biomedical imaging, remote sensing, bibliometrics, chemometrics and the
Internet.  We define high-dimensional tensor data to be
multi-dimensional data, with three or more dimensions or {\em tensor
  modes}, in which the dimension of one or 
more modes is large compared to the remainder.  Examples
of this include functional magnetic resonance imaging (fMRI) in which
three-dimensional images of the brain measured in voxels (when
vectorized, these form mode 1) are
taken over 
time (mode 2) for several experimental conditions (mode 3) and
for a set of independent subjects (mode 4). Here, the ``samples''
in this data set are the independent subjects ($\approx$ 30) or
experimental tasks ($\approx$ 20), 
the dimension of which are small compared to the number of voxels
($\approx$ 60,000) and time points ($\approx$ 2,000).  Exploratory
analysis, including dimension reduction, pattern recognition,
visualization, and feature selection, for this massive tensor data is
a major mathematical and computational challenge.  In this paper,
we propose flexible unsupervised multivariate methods, specifically
regularized Higher-Order Principal Components Analysis (HOPCA), to
understand and explore high-dimensional tensor data.

Classically, multivariate methods such as principal components
analysis (PCA) have been applied to matrix data
in which the number of independent observations $n$ is larger than the
number of variables $p$.  Exploring tensor data using these methods
would require flattening the tensor into a matrix with one dimension
enormous compared to the other.  This approach is not ideal
as to begin with, the important multi-dimensional structure of the tensor is
lost.  To solve this, many use traditional tensor decompositions such
as the CANDECOMP / PARAFAC (CP) \citep{harshman_1970,carroll_cp_1970}
and Tucker decompositions 
\citep{tucker_1966} to compute HOPCA \citep{kolda_tensor_review}.
These factorizations seek to 
approximate the tensor by a low-rank set of factors along each of the
tensor modes.  Unlike the singular value decomposition (SVD) for
matrix data, these tensor decompositions do not uniquely decompose the
 data \citep{kolda_tensor_review} further
complicating the analysis of 
tensor data.  Especially in the statistics community, there has been
relatively little work done on methodology for high-dimensional tensor data
\citep{mccullagh_tensor_1987, kroonenberg_2008, hoff_array_norm_2011}.

Many, especially in applied mathematics, have studied tensor
decompositions, but relatively few have advocated encouraging sparsity
or other types of regularization in the tensor factors.  The exception
is the non-negative tensor factorization, which is a well developed
area and has been used to cluster tensor data
\citep{cichocki_2009}.  In the 
context of these non-negative factorizations, several have
discussed sparsity \citep{hazan_2005, morup_sp_tucker_2008,
  lim_nonnegative_2009, cichocki_2009, liu_sp_nn_2011,
  chi_tensors_2011}, and a few have gone on to explicitly encourage  
sparsity in one of the tensor factors \citep{ruiters_2009,
  pang_2011}. 
A general framework for sparsity and other types of flexible
regularization in tensor factorizations and HOPCA is lacking.

Sparsity in tensor decompositions and HOPCA is desirable
for many reasons.  First, tensor decompositions are often used to
compress large multi-dimensional data sets
\citep{kolda_tensor_review}.  Sparsity in the tensor factors
compresses the data further and is attractive from a data storage
perspective.  Second, in high-dimensional settings, many features are
often irrelevant.  With fMRIs, for
example, there are often hundreds of thousands of voxels in each image
with many voxels inactive for much of the
scanning session.  Sparsity gives one an automatic tool for feature
selection. Third, many have
noted that PCA is asymptotically inconsistent in
high-dimensional settings \citep{johnstone_jasa_2009, jung_pca_2009}.
As this is 
true for matrix data, 
it is not hard to surmise that asymptotic inconsistency of the
factors holds for HOPCA as well.  Sparsity in PCA, however,
has been shown to yield consistent PC directions
\citep{johnstone_jasa_2009, amini_2009}.  Finally for high-dimensional
tensors, visualizing and 
interpreting the HOPCs can be a challenge.  Sparsity
limits the number of features and hence simplifies visualization and
interpretation of exploratory data analysis results.  Similar
arguments can be made for smoothing the tensor factors or encouraging
other types of regularization.

In this paper, we seek a mathematically sound and computationally
attractive framework for regularizing one, two, or all of the tensor
factors, or HOPCs.  Our
approach is based on optimization theory.  That is, we seek to find a
coherent optimization criterion for regularized HOPCA related to
traditional tensor decomposition approaches and an algorithm that
converges to a 
solution to this optimization problem.  Recently, many have proposed
similar optimization-based approaches to perform regularized PCA.
Beginning with \citet{jolliffe_spca_2003}, several have proposed to
find Sparse PCs by 
solving an optimization problem related the PCA with an additional
$\ell_{1}$-norm penalty \citep{zou_spca_2006, shen_spca_2008}.  Others
have shown that Functional 
PCA can be achieved by solving a PCA optimization problem with a
penalty to induce smoothness \citep{silverman_1996, huang_fpca_2008}.
More recently, several have 
extended these optimization approaches to two-way penalized SVDs by
placing penalties on both the row and column factors simultaneously
\citep{witten_pmd_2009, huang_twfpca_2009, lee_ssvd_2010,
  allen_gmd_2011}.  Specifically, our approach to regularized HOPCA
most 
closely follows the work of 
\citet{shen_spca_2008} and later \citet{witten_pmd_2009} and
\citet{allen_gmd_2011} who iteratively solve 
penalized regression problems converging to a rank-one solution with
subsequent factors found via a greedy deflation method.

To limit the scope of our exploration, we first develop a framework
for Sparse HOPCA and later show that extensions of this framework
yield methods for regularization with general penalties, Functional
HOPCA, and Generalized HOPCA for structured tensor data.  Additionally,
as the concept of regularization in HOPCA is relatively undeveloped,
we employ a less common strategy by first introducing
several flawed approaches to Sparse HOPCA.  We do this for a number of
reasons, namely to illuminate these problematic paths for future
researchers, to give us a basis for comparison for our
optimization-based methods, and to provide further background and
justification for our approach.  Our optimization-based approach
directly solves a multi-way concave relaxation of a rank-one HOPCA
method.  Overall, this method has many mathematical and computational
advantages including guaranteed convergence to a local optimum.

This paper is organized as follows.  As the general statistical
audience may be unfamiliar with tensors and tensor decompositions for
HOPCA, we begin by introducing notation, reviewing these
decompositions, and discussing these in the context of PCA in Section
\ref{section_prelim}.  We take a first step at considering how to
introduce sparsity in this context in Section \ref{section_heuristic}
by walking the reader through several novel (but, flawed!) algorithmic
approaches 
to Sparse HOPCA.  Then, in Section \ref{section_deflation}, we
introduce our main novel method and algorithm for Sparse HOPCA, based
on deflation and the
CP factorization.  In Section \ref{section_ext}, we outline several novel
extensions of this approach to allow flexible modeling with general
penalties and functional or structured tensor data.  Simulation
studies are presented in Section \ref{section_res} followed by two
case studies to high-dimensional three-way microarray data and an fMRI
study.  We conclude with a discussion, Section \ref{section_dis}, of
the implications of our work 
as well as the many directions for future work on tensor data.

\section{Preliminaries: Tensor Decompositions}
\label{section_prelim}

In this section, we review the two major approaches to tensor
decompositions and HOPCA, the CANDECOMP / PARAFAC
(CP) and Tucker decompositions.  We also introduce
notation and discuss
considerations for evaluating HOPCA such as the amount of variance
explained.

As notation with multi-way data can be cumbersome, we begin by reviewing
the notation used throughout this paper, mostly adopted from
\citet{kolda_tensor_review}.  Tensors will be
denoted as $\tenX$, matrices as $\X$, vectors as $\x$ and scalars as
$x$.  As there are many types of multiplication with tensors, the
outer product will be denoted by $\circ$, $\x \circ \y = \x
\y^{T}$. Specific dimensions of the tensor will be called modes and
multiplication by a matrix or vector along a tensor mode will be
denoted as $\times_{1}$; here, the subscript refers to the mode being
multiplied using regular matrix multiplication.  For example, if
$\tenX \in \Re^{n \times p \times q}$ and $\U \in \Re^{n \times K}$,
then $\tenX \times_{1} \U \in \Re^{K \times p \times q}$.
Sometimes it is
necessary to flatten the tensor into a matrix, or matricize the
tensor.  This is denoted as $\X_{(1)}$ where the subscript indicates the
mode along which the matricization has occurred.  For example, if $\tenX
\in \Re^{n \times p \times q}$, then $\X_{(1)} \in \Re^{n \times
  pq}$.  
The tensor Frobenius norm, $|| \tenX ||_{F}$, refers
to $|| \tenX  ||_{F} = \sqrt{ \sum_{i} \sum_{j} \sum_{k}
  \tenX_{ijk}^{2} }$.  Two types of matrix multiplication that will be
important in this paper are the Kronecker product and the perhaps less
familiar Khatri-Rao product.  The former is denoted by $\otimes$ and
the latter by $\odot$.  For matrices $\U \in \Re^{n \times K}$ and $\V
\in \Re^{p \times K}$, the Khatri-Rao product is defined as $\U \odot
\V = [ \uvec_{1} \otimes \vvec_{1} \ \ \ldots \ \ \uvec_{K} \otimes
  \vvec_{K} ]$.  For notational simplicity, all results in this
paper will be presented for the three-mode tensor.  Our methods
can all be trivially extended to multi-dimensional tensors with an
arbitrary number of modes.

As discussed in the introduction, there is no equivalent to the SVD
for multi-way data with three or more dimensions.  Thus, for tensor
data, people commonly use one of two tensor decompositions that
capture desirable aspects of the SVD.
The CP decomposition seeks to model a tensor as
a sum of rank one tensors, where rank one three-mode tensors are
defined as an outer product of three vectors: $\tenX \approx \sum_{k=1}^{K}
d_{k} \uvec_{k} 
\circ \vvec_{k} \circ \wvec_{k}$, where $\uvec_{k} \in \Re^{n}$,
$\vvec_{k} \in \Re^{p}$, $\wvec_{k} \in \Re^{q}$ and $d_{k} \geq 0$
\citep{harshman_1970,carroll_cp_1970}.  In this paper, we will assume
that the three vectors have norm one, i.e. $\uvec^{T} \uvec = 1$, and
define the factor matrices as $\U \in \Re^{n \times K} = [ \uvec_{1} \
  \ \ldots \ \ \uvec_{K} ]$ and so forth.  Thus, similar to the SVD,
the CP seeks to factorize the data into a sum 
of rank-one arrays, although unlike the SVD these rank-one arrays do
not uniquely decompose the data and in general, are not 
orthonormal.  The Tucker decomposition,
sometimes referred to as the Higher-Order SVD (HOSVD) or HOPCA, seeks to model
a three-mode tensor as $\tenX \approx \tenD \times_{1} \U
\times_{2} \V \times_{3} \W$ where the factors $\U \in \Re^{n \times
  K_{1}}$, $\V \in \Re^{p \times K_{2}}$ and $\W \in \Re^{q \times
  K_{3}}$ are orthonormal and $\tenD \in \Re^{K_{1} \times K_{2} \times
  K_{3}}$ is the core tensor \citep{tucker_1966, de_lathauwer_2000}.
Hence, the factors of the Tucker decomposition are orthonormal,
similar to those of the SVD, but these orthonormal factors do not
uniquely decompose the array into a sum of rank-one factors with a
diagonal core.

Both the CP and Tucker decompositions can be used for tensor data in a
similar manner to PCA for matrix data.  In other words, one can think
of the tensor factors as major modes of variation or patterns in the
data that can be used for exploratory analysis, visualization, and
dimension reduction.    A critical quantity in assessing the dimension
reduction 
achieved for PCA is the amount of variance explained by each of the
SVD factors.  As existing tensor factorizations and the methods we
will propose in this paper do not uniquely decompose the data array,
we cannot simply sum the equivalent of the squared singular values to
measure the variance explained.  Since the Tucker decomposition
imposes orthonormality on the factors, the proportion of variance
explained by the decomposition has a simple form: $|| \tenD ||_{F}^{2}
/ || \tenX ||_{F}^{2}$ \citep{de_lathauwer_2000, kolda_tensor_review}.
This is not the case for the CP decomposition where many have
erroneously referred the the variance explained by the $k^{th}$
component as $d^{2}_{k} / || \tenX ||_{F}^{2}$.  Instead, as the
factors of the CP decomposition may be correlated, we must project out
these effects to determine the proportion of variance explained:
\begin{theorem}
\label{thm_var_ex}
Define $\Pmat^{(U)}_{k} = \U_{k} ( \U_{k}^{T} \U_{k} )^{-1} \U_{k}^{T}$
where $\U_{k} = [ \uvec_{1}, \ldots \uvec_{k}]$ and define
$\Pmat^{(V)}_{k}$ and $\Pmat^{(W)}_{k}$ analogously.  Then, the cumulative
proportion of variance explained by the first $k$ HOPCs (or
regularized HOPCs) is given by $|| \tenX \times_{1} \Pmat^{(U)}_{k}
\times_{2} \Pmat^{(V)}_{k} \times_{3} 
   \Pmat^{(W)}_{k} ||_{F}^{2} \  /  \ || \tenX ||_{F}^{2}$.
\end{theorem}
As our regularized tensor decompositions presented in the next two
sections will not impose orthogonality on the factors, this result is
critical in evaluating the dimension reduction achieved by both the CP
decomposition and our regularized HOPCA methods.


\section{Algorithmic Approaches to Sparse HOPCA}
\label{section_heuristic}


We begin our exploration of methods for finding Sparse Higher-Order
Principal Components by proposing three algorithmic approaches.  These
methods are direct extensions of popular algorithms for the CP and
Tucker decompositions that one might first think of when seeking
sparsity.  We will show, however, that while these approaches are
intuitively simple, they are problematic both mathematically and
computationally as they do not solve a coherent mathematical
objective.  Despite this, these methods are important to present to
both steer future investigators from problematic paths and provide
background, rationale, and a basis for comparison to our 
subsequent optimization-based Sparse HOPCA methods.  Here,
we walk through each of these algorithmic approaches; outlines of the
full algorithms are given in the Supplemental Materials.

The Higher-Order SVD (HOSVD) and Higher-Order Orthogonal Iteration
(HOOI), sometimes referred to as the Tucker Alternating Least Squares
(Tucker-ALS), are two popular algorithms for computing the Tucker
decomposition \citep{tucker_1966,  de_lathauwer_2000,
  kolda_tensor_review}.  The former estimates each factor matrix by
calculating the singular vectors of the tensor matricized along each
mode \citep{tucker_1966,  de_lathauwer_2000}.  In other words for a
three-mode tensor, the HOSVD 
can be found by performing PCA three times on data flattened along
each of the three dimensions.  The HOOI seeks to improve upon this
decomposition by minimizing a Frobenius norm loss between $\tenX$ and
the Tucker decomposition $\tenD \times_{1} \U^{T} \times_{2} \V^{T}
\times_{3} \W^{T}$ subject to orthogonality constraints on the
factors.  \citet{de_lathauwer_2000} showed that this problem is solved by
an iterative algorithm where each Tucker factor is estimated by
computing the singular vector of the tensor projected onto the other
factors.  In other words, the HOOI updates $\U$ by performing PCA on
the matricized $( \tenX \times_{2} \V \times_{3} \W )_{(1)}$.  Thus,
both the HOSVD and HOOI algorithms estimate the factors by performing
PCA on a matricized tensor.  This leads to a simple strategy for
obtaining  Sparse HOSVD and Sparse HOOI: Replace PCA with Sparse PCA in
each algorithm to obtain sparse factors for each tensor mode.  Many
algorithms exist for performing Sparse PCA 
\citep{jolliffe_spca_2003, zou_spca_2006, shen_spca_2008,
  johnstone_jasa_2009, journee_spca_2010}, any of 
which can be used to compute the Sparse HOSVD or Sparse HOOI.

While these methods for Sparse HOPCA based on the Tucker decomposition
are conceptually simple, they are not ideal for several reasons.
First, the methods are purely heuristic algorithmic approaches and do
not solve any coherent mathematical optimization problem.  For the
Sparse HOOI method, this is particularly problematic as one cannot
guarantee convergence of the algorithm without additional
constraints.  Secondly in 
high-dimensional settings, matricizing the 
tensor along each mode and performing Sparse PCA is computationally
intensive and requires large amounts of computer memory.  Employing
the Sparse PCA methods of \citet{jolliffe_spca_2003, zou_spca_2006}
requires forming and computing the 
leading sparse eigenvalues of $n \times n$, $p
\times p$, and $q \times q$ matrices, which in high-dimensional
settings are typically much larger than the original data array.  The
methods of \citet{shen_spca_2008, journee_spca_2010} require computing
the sparse singular vectors of 
$n \times pq$, $p \times nq$, and $q \times np$, which corresponds to
calculating several unnecessary and extremely large singular vectors.
Hence, even though the Sparse HOSVD and HOOI are attractive in their
simplicity, mathematically and computationally they are less
desirable.

Next, we consider incorporating sparsity in the CP decomposition
through the popular CP Alternating Least Squares (CP-ALS) Algorithm.
This algorithm updates one factor at a time by minimizing a Frobenius
norm loss with the other factors fixed
\citep{harshman_1970,carroll_cp_1970, kolda_tensor_review}.  Consider,
for example, solving for $\U$ with $\V$ and $\W$ fixed.  Then, the
loss function can be written as $|| \tenX - \sum_{k=1}^{K} d_{k}
\uvec_{k} \circ \vvec_{k} \circ \wvec_{k} ||_{F}^{2} = || \X_{(1)} -
\hat{\U} (\V \odot \W )^{T} ||_{F}$, where $\hat{\U} = \U
\mathrm{diag}( \mathbf{d} )$ \citep{harshman_1970,carroll_cp_1970}.
Thus, the CP-ALS algorithm estimates 
each factor sequentially by performing simple least squares and
normalizing the columns of the factor matrix to have norm one,
i.e. $\hat{d}_{k} = || \hat{\uvec}_{k} ||$.  A simple strategy for
encouraging sparsity is then to replace the least squares problem with
an $\ell_{1}$-norm penalized least squares, or LASSO problem
\citep{tibshirani_1996}.  In other words, to estimate a sparse $\U$,
we minimize $|| \X_{(1)} - \hat{\U} (\V \odot \W )^{T} ||_{F} + \lamu
|| \hat{\U} ||_{1}$, where $\lamu$ is a non-negative regularization
parameter and $|| \cdot ||_{1} = \sum \sum | \cdot |$ is the
$\ell_{1}$-norm.  Hence, a possible method for Sparse HOPCA is to
replace each least squares update in the CP-ALS algorithm with a LASSO
update; we call this the Sparse CP-ALS method.

A first glance, it seems that the Sparse CP-ALS is less heuristic than
the Sparse HOSVD and HOOI methods.  As each update solves a LASSO
regression problem, it is natural to presume that the algorithm
jointly minimizes the Frobenius norm loss with $\ell_{1}$ penalties on
each of the factors:
\begin{align}
\label{not_cp_als}
\minimize_{\U, \V, \W} \ & \ \frac{1}{2} || \tenX - \sum_{k=1}^{K}
d_{k} \uvec_{k} \circ \vvec_{k} \circ \wvec_{k} ||_{F}^{2} + \lamu ||
\U ||_{1} + \lamv || \V ||_{1} + \lamw || \W ||_{1} \nonumber \\
\textrm{subject to} \ & \ d_{k} \geq 0, \ \uvec_{k}^{T} \uvec_{k} = 1,
\ \vvec_{k}^{T} 
\vvec_{k} = 1, \ \& \ \wvec_{k}^{T} \wvec_{k} = 1, \  \forall \ k=1,
\ldots K.
\end{align}
Interestingly, this is not the case!
\begin{proposition}
\label{prop_cp_als}
The Sparse CP-ALS algorithm does not minimize \eqref{not_cp_als}.
\end{proposition}
In fact, it can be shown that the Sparse CP-ALS does not optimize any
coherent objective.  Similar results, while not proved explicitly,
have been implied for two-way penalized SVDs \citep{witten_pmd_2009,
  lee_ssvd_2010, allen_gmd_2011}.  Thus, similar to the Sparse HOSVD
and HOOI methods, convergence of the Sparse CP-ALS algorithm cannot be
guaranteed without further constraints.  In the next section, we
present a framework for incorporating sparsity in the CP decomposition
with provable results in terms of convergence to a solution of an
optimization problem.

\section{Deflation Approach to Sparse HOPCA}
\label{section_deflation}
In this section, we develop a novel deflation approach to Sparse HOPCA
that find each rank-one factorization by solving a tri-concave
relaxation of the CP optimization problem.  But first, we show how the
CP decomposition can be obtained by an alternative greedy algorithm,
the deflation-based Tensor Power Algorithm.

\subsection{Tensor Power Algorithm}
\label{section_tpa}

We introduce an alternative form of the rank-one CP optimization problem and a
corresponding algorithm that will form the foundation of our deflation
approach to Sparse HOPCA.  The single-factor CP decomposition solves
the following 
optimization problem \citep{kolda_tensor_review}: 
\begin{align}
\minimize_{\uvec, \vvec, \wvec, d} \ \ || \tenX - d \uvec \circ
\vvec \circ \wvec ||_{2}^{2}  \ \
\textrm{subject to} \  \ \uvec^{T} \uvec = 1, \vvec^{T} \vvec = 1,
\wvec^{T} \wvec = 1, \ \& \ d > 0. 
\end{align}
Some algebra manipulation (see \citet{kolda_tensor_review}) shows
that an equivalent form to this optimization problem is given by the
following: 
\begin{align}
\label{cp_max}
\maximize_{\uvec, \vvec, \wvec} \ \ \tenX \times_{1} \uvec
\times_{2} \vvec \times_{3} \wvec  \ \
\textrm{subject to} \  \ \uvec^{T} \uvec = 1, \vvec^{T} \vvec = 1,
\ \& \ \wvec^{T} \wvec = 1.
\end{align}
As \eqref{cp_max} is separable in the factors, we can optimize this in
an iterative block-wise manner:   
\begin{proposition}
\label{prop_tpa}
The block coordinate-wise solutions for \eqref{cp_max} are given by:
{\small \begin{align*}
\hat{\uvec} =  \frac{\tenX \times_{2} \vvec \times_{3}
  \wvec}{ || \tenX \times_{2} \vvec \times_{3}
  \wvec ||_{2}}, \hat{\vvec} =  \frac{\tenX \times_{1} \uvec \times_{3}
  \wvec}{ || \tenX \times_{1} \uvec \times_{3}
  \wvec ||_{2}},  \ \& \ \hat{\wvec} =  \frac{\tenX \times_{1} \uvec \times_{2}
  \vvec}{ || \tenX \times_{1} \uvec \times_{2}
  \vvec ||_{2}}. 
\end{align*}}
Updated iteratively, these converge to a local optimum of \eqref{cp_max}.
\end{proposition}
As each coordinate update increases the objective and the objective is
bounded above by $d$, convergence of this scheme is assured.  
Note, however, that this approach only converges to a
local optimum of \eqref{cp_max}, but this is true of all other
algorithmic approaches to  
solving the CP problem as well \citep{kolda_tensor_review}.

To compute multiple CP factors, one could apply this single-factor
approach sequentially to the residuals remaining after subtracting
out the previously computed factors.   This so called deflation
approach is 
closely related in structure to the power method for computing
eigenvectors \citep{golub_van_loan_1996}.  We then call this greedy method
the {\em Tensor Power Algorithm}.
Notice that this method does not enforce orthogonality
in subsequently computed components.  The algorithm can be easily
modified, however, to ensure orthogonality by employing a Graham-Schmidt
scheme \citep{golub_van_loan_1996}.

Before moving on to our Sparse CP method, we pause to discuss the
Tensor Power Method and compare it to more common algorithms to
compute the CP decomposition such as
the Alternating Least Squares algorithm
\citep{harshman_1970,carroll_cp_1970}. As the Tensor Power Algorithm is a
greedy approach, the first several factors computed will tend to
explain the most variance in the data.  In contrast, the CP-ALS
algorithm seeks to maximize the set of $d_{k}$ for all $K$ factors
simultaneously.  Thus, the set of $K$ CP-ALS factors may explain as much
variance as those of the Tensor Power Algorithm, but the first several
factors typically explain much less.  This is illustrated on the tensor
microarray data in Section \ref{section_res}.  Also, while the
CP-ALS algorithm returns $d_{k}$ in descending order, the $d_{k}$
computed via the Tensor Power Method are not necessarily ordered.

\subsection{Sparse HOPCA via Deflation}

We introduce a novel deflation-based method for Sparse HOPCA that
incorporates 
sparsity by regularizing factors with an $\ell_{1}$-norm penalty.
Our method solves a direct relaxation of the CP optimization problem
\eqref{cp_max} and has a computationally attractive solution.

Many have sought to encourage sparsity in PCA by greedily solving
$\ell_{1}$-norm penalized optimization problems 
related the to rank-one SVD \citep{jolliffe_spca_2003, zou_spca_2006,
  shen_spca_2008, witten_pmd_2009, lee_ssvd_2010, journee_spca_2010,
  allen_gmd_2011}.  We formulate our method similarly by
placing $\ell_{1}$-norm penalties on each of the tensor factors of the
CP optimization problem 
\eqref{cp_max} and relaxing the equality constraints to inequalities.
In this manner, our approach is most similar to the two-way Sparse PCA
methods of \citet{witten_pmd_2009} and \citet{allen_gmd_2011}.
We define our single-factor Sparse CP-TPA  
(named to denote its relation to our Tensor Power Algorithm)
factorization as the solution to the following problem:
\begin{align}
\label{sparse_cp_prob}
\maximize_{\uvec, \vvec, \wvec} \ & \ \tenX \times_{1} \uvec
\times_{2} \vvec \times_{3} \wvec    - \lamu || \uvec ||_{1} -
\lamv || \vvec ||_{1} - \lamw || \wvec ||_{1} \nonumber \\
\textrm{subject to} \ & \ \uvec^{T} \uvec \leq 1, \vvec^{T} \vvec \leq 1,
\ \& \ \wvec^{T} \wvec \leq 1.
\end{align}
Here $\lamu$ , $\lamv$ and $\lamw$ are
non-negative regularization parameters controlling the amount of sparsity in
the tensor factors.  Relaxing the constraints greatly simplifies
the optimization problem, leading to a simple solution and algorithmic
approach:
\begin{theorem}
\label{thm_sparse_cp}
Let $\hat{\uvec} = S (\tenX \times_{2} \vvec \times_{3}
  \wvec, \lamu )$, $\hat{\vvec} = S (\tenX \times_{1} \uvec \times_{3}
  \wvec, \lamv )$, and $\hat{\wvec} = S (\tenX \times_{1} \uvec \times_{2}
  \vvec, \lamw )$ where $S(\cdot, \lambda) = \mathrm{sign}(\cdot)(|
  \cdot | - \lambda )_{+}$ is the soft-thresholding operator.   Then,
  the factor-wise solutions to the Sparse CP-TPA problem are
  given by: 
{\small \begin{align*}
\uvec^{*} =  \begin{cases} \frac{ \hat{\uvec}  }{ || \hat{\uvec} ||_{2}} &
  || \hat{\uvec} ||_{2} > 0 \\ 0 & \textrm{otherwise}, 
\end{cases}
\vvec^{*} =  \begin{cases} \frac{ \hat{\vvec}  }{ || \hat{\vvec} ||_{2}} &
  || \hat{\vvec} ||_{2} > 0 \\ 0 & \textrm{otherwise}, 
\end{cases}
\& \wvec^{*} =  \begin{cases} \frac{ \hat{\wvec}  }{ || \hat{\wvec} ||_{2}} &
  || \hat{\wvec} ||_{2} > 0 \\ 0 & \textrm{otherwise}, 
\end{cases} 
\end{align*}}
Each factor-wise update monotonically increases the objective and
when iterated, they converge to a local maximum of \eqref{sparse_cp_prob}.
\end{theorem}
Thus, \eqref{sparse_cp_prob} can be solved a factor at a time by
soft-thresholding and then re-scaling the update.  While this type of
algorithmic scheme may seem familiar to some, a major contribution of
our work is determining the underlying optimization problem being solved
by this algorithmic structure.

There are many mathematical and computational advantages to this
approach.  First, notice that \eqref{sparse_cp_prob} is a concave
optimization problem in each factor with the other factors fixed.
Thus, it is a multi-way concave optimization problem, meaning that we
can solve for one factor at a time with the others fixed.  This scheme
yields a monotonic algorithm that converges to a local maximum of
\eqref{sparse_cp_prob} \citep{tseng_2001}.  This is important as, unlike
our algorithmic based Sparse HOPCA methods, the progress of the
algorithm can be tracked and convergence is assured.  Second, notice
that in \eqref{sparse_cp_prob}, the scale of the factors is directly
controlled.  Even with the relaxed constraints, the solution for each
factor is guaranteed to either have norm one or be set to zero.  This
is a major computational advantage as it leads to numerically
well-conditioned algorithmic updates and solutions.  Third, this
method is computationally attractive as each factor-wise update has a
simple analytical solution, and thus each update is computationally
inexpensive.  Additionally, this scheme requires far less computer
memory than our algorithmic Sparse HOPCA approaches as only factors
comprising the 
final solution need to be computed, and the tensor never needs to be
matricized.  The final advantage of this approach is its generality;
as we will show in the next section, there are many possible
extensions of this methodology.

To compute multiple factors, we employ a deflation approach as in the
Tensor Power Algorithm.  
Specifically, each factor is calculated by solving the single-factor CP
problem, \eqref{sparse_cp_prob}, for the residuals from the previously
computed single-factor solutions.  
Notice that we do not enforce orthogonality in the
factors.  In fact, many have advocated against enforcing orthogonality of
sparse PCs \citep{zou_spca_2006, shen_spca_2008,
  journee_spca_2010}.  For factors with $\lambda = 0$, however, if
orthogonality is desired, this can be accomplished by altering the
factor updates as described in Section \ref{section_tpa}.  Thus, we
have developed a deflation approach to Sparse HOPCA by greedily
solving single-factor optimization problems related to the CP
criterion.  This Sparse CP-TPA method enjoys mathematical advantages
such as assured convergence to a local optimum and is computationally
attractive for high-dimensional tensors.

\subsubsection{Selecting Dimension and Regularization Parameters}

Another important item of consideration for our Sparse CP-TPA method
and all of our Sparse HOPCA
methods is the number of factors, $K$, to compute for a given data
sets.  Several have proposed heuristic methods for choosing $K$ in
HOPCA (see \citet{kolda_tensor_review} and \citet{kroonenberg_2008}
and references therein). 
While several non-heuristic methods exist for PCA on matrix data
\citep{buja_1992, owen_2009_cv},
extending these to the tensor framework is non-trivial and beyond the
scope of this paper.

Our Sparse CP-TPA problem has three regularization parameters that
control the amount of sparsity in the factors.  Several methods
exist for selecting these regularization parameters in the Sparse PCA
literature \citep{troyanskaya_2001, owen_2009_cv,
  shen_spca_2008, lee_ssvd_2010}. As cross-validation can be slow to
run for high-dimensional tensors, we choose to select regularization
parameters via the Bayesian Information Criterion (BIC)
\citep{lee_ssvd_2010, allen_gmd_2011}:
$\lambda_{\uvec}^{*} = \argmin_{\lambda_{\uvec}} BIC( \lambda_{\uvec} )$ where 
$BIC(\lambda_{\uvec} ) = \mathrm{log}\left(\frac{|| \tenX - d \uvec \circ
  \vvec \circ \wvec ||_{F}^{2} }{npq}\right) +
\frac{\mathrm{log}(npq)}{npq} | \{ \uvec \} |$, where $| \{ \uvec\} |$
is the number of non-zero elements of $\uvec$.  This BIC formulation
can be derived by considering that each update of the Sparse CP-TPA
algorithm solves an $\ell_{1}$-norm penalized regression
problem. Selection criteria for $\vvec$ and $\wvec$ are analogous.
Similar BIC-based selection methods can be derived for our previous
algorithmic Sparse HOPCA methods as well.
Experimental results evaluating the efficacy of this regularization parameter
selection method are given in Section \ref{section_sims}.

\section{Extensions}
\label{section_ext}

As our frameworks introduced for Sparse HOPCA
are general, there are many possible extensions of our methodology.
In this section, we outline novel results and methodology for three
extensions we consider most important: general penalties and
non-negativity, generalizations for structured tensors, and multi-way
functional HOPCA.  Our intent here is to refrain from fully describing
each development, but instead to give the reader enough details to
understand and use each of these extensions.

\subsection{General Penalties \& Non-negativity}

In certain applications, one may wish to regularize tensor factors
with a penalty other than an $\ell_{1}$-norm.  Consider the following
optimization problem which incorporates general penalties,
$P_{\uvec}()$, $P_{\vvec}()$ and $P_{\wvec}()$:
\begin{align}
\label{cp_prob_reg}
\maximize_{\uvec, \vvec, \wvec} \ & \ \tenX \times_{1} \uvec
\times_{2} \vvec \times_{3} \wvec   - \lamu P_{\uvec}( \uvec ) -
\lamv P_{\vvec}( \vvec ) - \lamw P_{\wvec}( \wvec ) \nonumber \\
\textrm{subject to} \ & \ \uvec^{T} \uvec \leq 1, \vvec^{T} \vvec \leq 1,
\ \& \ \wvec^{T} \wvec \leq 1.
\end{align}
An extension of a
result in \citet{allen_gmd_2011} reveals that one may solve this
optimization problem for general penalties that are convex and order
one by solving penalized regression problems:
\begin{theorem}
\label{thm_cp_gen_pen}
Let $P_{\uvec}()$, $P_{\vvec}()$ and $P_{\wvec}()$ be convex and
homogeneous or order one.  Consider the following penalized regression
problems:
$\hat{\uvec} = argmin_{\uvec} \{ \frac{1}{2} || \tenX \times_{2}
\hat{\vvec} \times_{3} \hat{\wvec} - \uvec ||_{2}^{2} + \lamu
P_{\uvec}( \uvec ) \}$, $\hat{\vvec} = argmin_{\vvec} \{ \frac{1}{2} || \tenX \times_{1}
\hat{\uvec} \times_{3} \hat{\wvec} - \vvec ||_{2}^{2} + \lamv
P_{\vvec}( \vvec ) \}$, and
$\hat{\wvec} = argmin_{\wvec} \{ \frac{1}{2} || \tenX \times_{1}
\hat{\uvec} \times_{2} \hat{\vvec} - \wvec ||_{2}^{2} + \lamw
P_{\wvec}(\wvec ) \}$.
Then, the block coordinate-wise solutions for
\eqref{cp_prob_reg} are given by:
{\small
\begin{align*}
\uvec^{*} = \begin{cases} \frac{\hat{\uvec}}{||\hat{\uvec} ||_{2}} &
  || \hat{\uvec} ||_{2} > 0 \\ 0 & \textrm{otherwise} \end{cases},  \vvec^{*} = \begin{cases} \frac{\hat{\vvec}}{||\hat{\vvec} ||_{2}} &
  || \hat{\vvec} ||_{2} > 0 \\ 0 & \textrm{otherwise} \end{cases}, 
\& 
 \wvec^{*} = \begin{cases} \frac{\hat{\wvec}}{||\hat{\wvec} ||_{2}} &
  || \hat{\wvec} ||_{2} > 0 \\ 0 & \textrm{otherwise} \end{cases}. \hspace{8mm}
\end{align*}}
Iteratively updating these factors converges to a local optimum of
\eqref{cp_prob_reg}. 
\end{theorem}
An example of a possible penalty type of interest in many tensor
applications is the group lasso,
which encourages sparsity in groups of variables \citep{yuan_2006_group}.

Much attention in the literature has been given to the non-negative
and sparse non-negative tensor decompositions
\citep{hazan_2005, shashua_2005, morup_sp_tucker_2008, cichocki_2009,
  liu_sp_nn_2011}. These  
techniques have been used for multi-way clustering of tensor data. A
simple modification of Theorem \ref{thm_cp_gen_pen} allows us to solve
\eqref{sparse_cp_prob} when non-negativity constraints are added for
each factor: Replace the soft-thresholding function $S( x, \rho) =
\mathrm{sign}( x) ( | x | - \rho )_{+}$
with the positive-thresholding function $P( x, \rho) = ( x -
\rho )_{+}$ \citep{allen_sp_nn_gpca_2011}.  This, then, is a
computationally attractive alternative to estimating sparse
non-negative tensor factors.  Extensions of our deflation approach for
Sparse CP-TPA  allow for regularized HOPCA with general
penalties and non-negativity constraints.

\subsection{Generalized HOPCA}

For structured data, \citet{allen_gmd_2011} showed that generalizing
PCA by working with an alternative matrix norm capturing the known
structure can dramatically improve results.  The same techniques
generalizing PCA can be used to generalized HOPCA for structured
tensors.  Examples of these include image data as seen in
neuroimaging, microscopy and hyper-spectral imaging as well as
multi-dimensional NMR spectroscopy, remote sensing, and
environmetrics.  Here, we illustrate how to extend Generalized PCA to
the framework introduced for Sparse CP-TPA.
Generalized Tucker, Sparse HOSVD, and Sparse HOOI methods are a
straightforward extension of these approaches.

The Generalized PCA optimization problem of \citet{allen_gmd_2011}
replaces minimizing a Frobenius norm with a transposable-quadratic
norm comprised of two quadratic operators in the row and column space
of the data matrix.  To generalize a three-way Frobenius norm, let
$\Q^{(1)} \in \Re^{n \times n}$, $\Q^{(2)} \in \Re^{p \times p}$ and
$\Q^{(3)} \in \Re^{q \times q}$ and define the three-way quadratic
norm as the following: $|| \tenX ||_{\Q^{(1)},\Q^{(2)},\Q^{(3)}} =
\left( \sum_{i=1}^{n} 
   \sum_{i'=1}^{n} \sum_{j=1}^{p} \sum_{j'=1}^{p} \sum_{k=1}^{q}
   \sum_{k'=1}^{q} \Q^{(1)}_{ii'} \Q^{(2)}_{jj'} \Q^{(3)}_{kk'}
   \tenX_{ijk} \tenX_{i'j'k'} \right)^{1/2} = \left( \sum_{i=1}^{n}
   \sum_{j=1}^{p} 
   \sum_{k=1}^{q} \tilde{\tenX}_{ijk} \tenX_{ijk} \right)^{1/2}$,
where $\tilde{\tenX} = \tenX \times_{1} \Q^{(1)} \times_{2} \Q^{(2)}
\times_{3} \Q^{(3)}$. A major assumption implied by this three-way
quadratic norm is separability of the structured errors along each of
the tensor modes.  In fact, one can show that this norm implies a
three-way Kronecker covariance structure related to the array-normal
distribution \citep{hoff_array_norm_2011}. Give this, we define the
rank-one Generalized CP 
decomposition as the solution to the following optimization problem:
\begin{align}
\label{gcp}
\minimize_{\uvec,\vvec,\wvec,d} \ &  \ \frac{1}{2} || \tenX - 
d \uvec \circ \vvec \circ \wvec
||_{\Q^{(1)},\Q^{(2)},\Q^{(3)}}^{2} \nonumber \\
  \textrm{subject to} \ & \ 
\uvec^{T} \Q^{(1)} \uvec = 1, \ \vvec^{T} \Q^{(2)} \vvec = 1, \
\wvec^{T} \Q^{(3)} \wvec = 1, \ \& \ d > 0.
\end{align}
In the same manner in which we motivated the Sparse CP-TPA
optimization problem, we can define the rank-one Sparse Generalized
CP-TPA  decomposition to be the solution to the following:
\begin{align}
\label{sparse_gcp}
\maximize_{\uvec,\vvec,\wvec,} \ & \ \tenX \times_{1} \Q^{(1)} \uvec
\times_{2} \Q^{(2)} \vvec \times_{3} \Q^{(3)} \wvec  - \lamu || \uvec
||_{1} - \lamv || \vvec ||_{1} - \lamw || 
\wvec ||_{1} \nonumber \\
 \textrm{subject to} \ & \ \uvec^{T} \Q^{(1)} \uvec \leq 1, \
 \vvec^{T} \Q^{(2)} \vvec \leq  1, \ \& \ 
\wvec^{T} \Q^{(3)} \wvec \leq 1.
\end{align}

Following from results in \citep{allen_gmd_2011}, these rank-one
problems have closed form solutions.  
\begin{proposition}
\label{prop_ghopca}
The coordinate-wise updates for $\uvec$ are:
\begin{enumerate}
\item Generalized CP: $\uvec^{*} = ( \tenX \times_{2} \Q^{(2)} \vvec
  \times_{3} \Q^{(3)} \wvec ) / || ( \tenX \times_{2} \Q^{(2)} \vvec
  \times_{3} \Q^{(3)} \wvec ) ||_{\Q^{(1)}}$.
\item Sparse Generalized CP: Define $\hat{\uvec} = argmin_{\uvec} \{
  \frac{1}{2} || \tenX \times_{2} 
\Q^{(2)} \vvec \times_{3} \Q^{(3)} \wvec - \uvec ||_{\Q^{(1)}}^{2} +
  \lamu || \uvec ||_{1} \}$, then $\uvec^{*} = \hat{\uvec} / ||
  \hat{\uvec} ||_{\Q^{(1)}}$ if $|| \hat{\uvec} ||_{\Q^{(1)}} > 0$, and
  $\uvec^{*} = 0$ otherwise.  
\end{enumerate}
The updates for $\vvec$ and $\wvec$ are analogous.  Together these
updates converge to a local solution to the Generalized CP or Sparse
Generalized CP problem.  
\end{proposition}
As with our TPA and Sparse CP-TPA methods, subsequent components can
be calculated via deflation.  Thus, we have outlined a generalization
of our results on HOPCA 
and Sparse HOPCA for use with structured tensor data.

\subsection{Multi-way Functional PCA}

Many examples of high-dimensional tensors have functional valued
elements.  Examples include multi-dimensional spectroscopy data
commonly studied in chemometrics, series of images measured over time
such as in neuroimaging, and
hyper-spectral imaging data.  For matrix data, many have used Functional
PCA (FPCA) to reduced the dimension and discover patterns in functional
valued data \citep{silverman_1996}.  Recently,
\citet{huang_twfpca_2009} extended the 
Functional PCA framework to encompass two-way functional data.  These same
techniques can be further extended to multi-way data to study the
examples of functional tensor data described above.

\citet{huang_twfpca_2009} elegantly showed that estimating functional principal
components by solving a bi-convex optimization problem and using
deflation is 
equivalent to two-way half-smoothing the data, an extension of the
FPCA half-smoothing result of \citet{silverman_1996}.  Both the optimization
approach and half-smoothing approach can be extended for multi-way
tensor data.  For tensors, however, these two approaches do not yield
equivalent tensor factorizations.  We briefly summarize these two approaches
and their resulting properties.   Consider functional data that has
been discretized, yielding the tensor, $\tenX \in \Re^{n \times p
  \times q}$.  Define $\Su = \mathbf{I}_{(n)} +
\Omeg_{\mathbf{\uvec}}$, $\Sv = \mathbf{I}_{(p)} +
\Omeg_{\mathbf{\vvec}}$, and $\Sw = \mathbf{I}_{(q)} +
\Omeg_{\mathbf{\wvec}}$, where $\Omeg$ is a smoothing matrix
commonly used in FPCA; examples include the squared second or fourth
differences matrices \citep{ramsay_2006}.  Then, we can define Tensor FPCA via
half-smoothing as follows: (1) Half smooth the data, $\tilde{\tenX}
\leftarrow \tenX \times_{1} \Su^{-1/2} \times_{2} \Sv^{-1/2}
\times_{3} \Sw^{-1/2}$; (2) Take the Tucker decomposition of
$\tilde{\tenX} \approx \tilde{\tenD} \times_{1} \tilde{\U} \times_{2}
\tilde{\V} \times_{3} \tilde{\W}$; (3) and half-smooth the components,
$\uvec_{k} \leftarrow \Su^{-1/2} \tilde{\uvec}_{k} $, $\vvec_{k}
\leftarrow \Sv^{-1/2} \tilde{\vvec}_{k}$, and $\wvec_{k} \leftarrow
\Sw^{-1/2} \tilde{\wvec}_{k}$.  We can also define the rank-one Tensor
FPCA as the solution to the following tri-convex optimization
problem:
\begin{align}
\label{tensor_fpca}
\minimize_{\uvec, \vvec, \wvec} \ \ || \tenX - \uvec \circ \vvec \circ
\wvec ||_{F}^{2} + \uvec^{T} \Su \uvec \vvec^{T} \Sv \vvec \wvec^{T}
\Sw \wvec - || \uvec ||_{2}^{2} || \vvec ||_{2}^{2} || \wvec ||_{2}^{2}.
\end{align}
Subsequent components can be calculated via deflation.  
\begin{proposition}
\label{prop_hofpca}
\begin{enumerate}
\item The coordinate-wise updates converging to a local minimum of
  \eqref{tensor_fpca} are given by: $\uvec =  \Su^{-1} \left( \tenX \times_{2} \vvec \times_{3} \wvec
\right)  /  ( \vvec^{T} \Sv \vvec \wvec^{T} \Sw \wvec )$,
with the updates for $\vvec$ and $\wvec$ are defined analogously. 
\item The solution to \eqref{tensor_fpca} is not equivalent to Tensor
  FPCA via half-smoothing.  In other words, the latter is not
  necessarily a local minimum to \eqref{tensor_fpca}.
\end{enumerate}
\end{proposition}
Although half-smoothing and the solution to \eqref{tensor_fpca} are
not equivalent, they both solve their respective optimization problems
and thus the problem of non-convergence as with Sparse HOOI is
avoided.  Thus, methods for Tensor FPCA are another possible extension
of our frameworks developed for Sparse HOPCA.

\section{Results}
\label{section_res}

We compare the performance of the various methods introduced for HOPCA
and Sparse HOPCA in terms of signal recovery and feature
selection on simulated data and in terms of dimension reduction and
feature selection on a microarray and functional MRI example. 

\subsection{Simulation Studies}
\label{section_sims}

\begin{figure}[!!t]
\begin{center}
\includegraphics[width=6.5in,clip=true,trim=1in 0in 1in 0in]{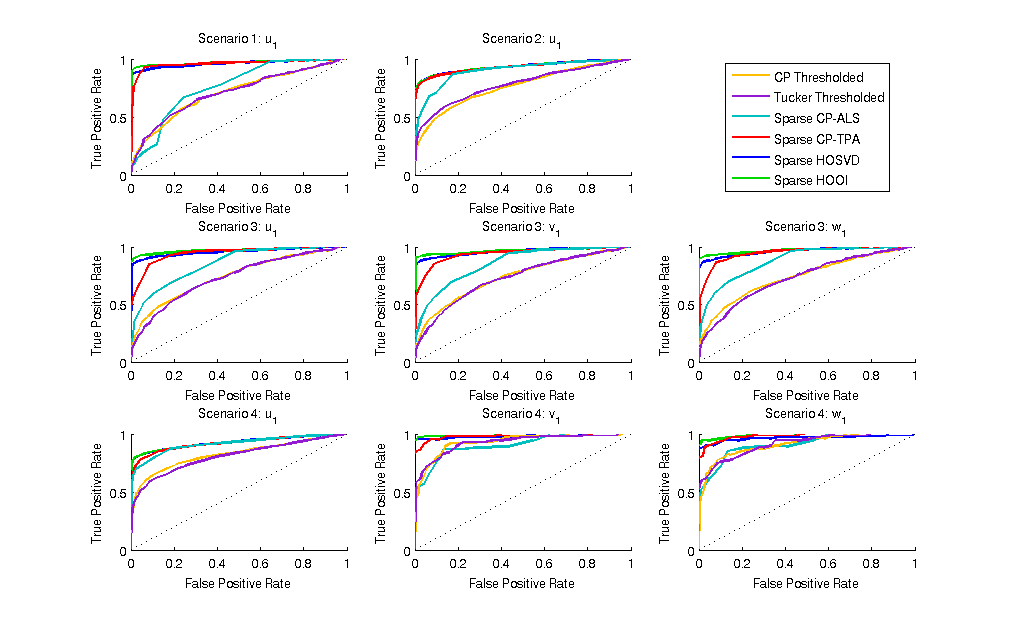}
\caption{\em \footnotesize ROC curves averaged over 50 replicates for
  each of the four simulation scenarios described in Section
  ~\ref{section_sims}.  The Sparse CP-TPA method performs similarly to
  Sparse HOSVD and Sparse HOOI in all but Scenario 3 where the later two
  have slight advantages.  The Sparse
  CP-ALS method, on the other hand, exhibits poor recovery of the true
  features, with performance often comparable to naive thresholding.}
\label{fig_roc}
\end{center}
\end{figure}

We evaluate the comparative performance of our methods on a simulated
low rank 
three-way tensor model comprised of sparse rank-one factors.  All data
is simulated from the following model: $\tenX = \sum_{k=1}^{K} d_{k} \uvec_{k} \circ
\vvec_{k} \circ \wvec_{k} + \tenE$, where the factors $\uvec_{k},
\vvec_{k}$, and $\wvec_{k}$ are random, $d_{k}$ is fixed and
$\tenE_{i,j,l}  \overset{iid}{\sim} N(0, 1)$. Four scenarios are
simulated and summarized as follows.  Scenario 1: $100 \times 100
\times 100$ with $\U$ sparse; Scenario 2: $1000 \times 20 \times 20$
with $\U$ sparse;  Scenario 3: $100 \times 100 \times 100$ with $\U$, 
$\V$ and $\W$ sparse; and Scenario 4: $1000 \times 20 \times 20$ with
$\U$, $\V$ and $\W$ sparse. 
Sparse factors are simulated with 50\% randomly selected elements set
to zero and non-zero values are i.i.d. $N(0,1)$. Non-sparse factors are
simulated as the first $K$ left and right singular vectors of a data
matrix with i.i.d. $N(0,1)$ entries.  In simulations where $K=1$,
$d_{1} = 100$.  In simulations where $K = 2$, $d_{1} = 200$ and $d_{2}
= 100$.  Additional simulation results for these four scenarios with
various signal to noise levels and a differing percentages of
sparse elements are given in the Supplemental Materials.

First, we test the accuracy of our methods in selecting the correct
non-zero features.  Receiver-operator curves (ROC) averaged over fifty
replicates 
computed by varying the regularization parameter, $\lambda$, 
are given for each of the four scenarios in Figure ~\ref{fig_roc}.  We
compare the Sparse HOSVD, Sparse HOOI, Sparse CP-ALS, and Sparse
CP-TPA to naive thresholding of the CP and Tucker decompositions that
act as a baseline.  In this paper, we implement the Sparse HOSVD and
Sparse HOOI using the Sparse PCA method of \citet{shen_spca_2008}.
From these comparisons, we see that the Sparse HOSVD and Sparse HOOI
consistently perform the best across all four simulation scenarios.
The Sparse CP-TPA method performs equally well in all but Scenario 3
where all factors are sparse and the samples sizes are equal.  The
Sparse CP-ALS method has considerably worse performance and barely
bests the naive thresholding methods.

\begin{table*}[!!t]
\scalebox{.8}{
\begin{tabular}{l|c|c|c|c|c|c}
\hline
& \multicolumn{6}{c}{Scenario 1} \\
\hline
& CP & Tucker & Sparse CP-ALS & Sparse CP-TPA & Sparse HOSVD & Sparse
HOOI \\
\hline
TP / FP $\uvec_{1}$ &   -   & - &  0.8308 / 0.8108  &  0.9332 / 0.0568
& 0.9372 / 0.0720   & 0.8472 / 0.0000 \\ 
TP / FP $\uvec_{2}$ & -    & - &     0.8204 / 0.8304   & 0.8688 /
0.0324 &   0.8960 / 0.2060 &   0.8704 / 0.0304 \\
Signal Recovery $\hat{\tenX}$ &   1.0052  &  1.0053  &  0.9993 &
0.0504  &  0.0501  &  0.0500 \\
\hline
& \multicolumn{6}{c}{Scenario 2} \\
\hline
& CP & Tucker & Sparse CP-ALS & Sparse CP-TPA & Sparse HOSVD & Sparse
HOOI \\
\hline
TP / FP $\uvec_{1}$ &   -   & - &      0.9310 / 0.2149  &  0.8874 /
0.0186 &   0.8902 / 0.0197  &  0.8346 / 0.0002 \\
TP / FP $\uvec_{2}$ &   -   & - &       0.8707 / 0.2366  &  0.7373 /
0.0329 &   0.6954 / 0.0066 &   0.7805 / 0.0197  \\
Signal Recovery $\hat{\tenX}$ &        1.0068  &  1.0032   & 0.9981 &
0.1239  &  0.1236  &  0.1236 \\
\hline
& \multicolumn{6}{c}{Scenario 3} \\
\hline
& CP & Tucker & Sparse CP-ALS & Sparse CP-TPA & Sparse HOSVD & Sparse
HOOI \\
\hline
TP / FP $\uvec_{1}$ & -   & - &     0.9832 / 0.7068  &  0.9468 /
0.1620 &   0.9500 / 0.0776  &  0.9108 / 0.0012 \\
TP / FP $\uvec_{2}$ & -   & - &     0.9732 / 0.7008  &  0.9116 /
0.2380  &  0.9008 / 0.2072  &  0.8940 / 0.0732  \\
TP / FP $\vvec_{1}$ & -   & - &     0.9840 / 0.7000   & 0.9412/ 0.1696
&  0.9420 / 0.0784  &  0.9024 / 0.0012 \\
TP / FP $\vvec_{2}$ & -   & - &     0.9676 / 0.6872  &  0.9152 /
0.2392  &  0.9000 / 0.2220   & 0.8916 / 0.0748 \\
TP / FP $\wvec_{1}$ & -   & - &     0.9860 / 0.6868  &  0.9460 /
0.1684  &  0.9436 / 0.0744  &  0.9112 / 0.0000 \\
TP / FP $\wvec_{2}$ & -   & - &     0.9740 / 0.7148   & 0.9140 /
0.2524 &   0.9092 / 0.2484  &  0.8888 / 0.0612 \\
Signal Recovery $\hat{\tenX}$ &        1.0066  &  1.0060   & 0.9994 &
0.0503 &   0.0499 &   0.0496  \\
\hline
& \multicolumn{6}{c}{Scenario 4} \\
\hline
& CP & Tucker & Sparse CP-ALS & Sparse CP-TPA & Sparse HOSVD & Sparse
HOOI \\
\hline
TP / FP $\uvec_{1}$ & -   & - &      0.8972 / 0.1618  &  0.8617 / 0.0256  &  0.8876 / 0.0184  &  0.8335 / 0.0002 \\
TP / FP $\uvec_{2}$ & -   & - &     0.7987 / 0.1838  &  0.7986 / 0.1455  &  0.7083 / 0.0067 &   0.7860 / 0.0190 \\
TP / FP $\vvec_{1}$ & -   & - &       0.9780 / 0.5220  &  0.9320 / 0.0580  &  0.9760 / 0.1440  &  0.9400 / 0.0000 \\
TP / FP $\vvec_{2}$ & -   & - &        0.9400 / 0.4720  &  0.9080 / 0.1880   & 0.9700 / 0.5620  &  0.9520 / 0.1360 \\
TP / FP $\wvec_{1}$ & -   & - &      0.9840 / 0.4560  &  0.9260 / 0.0620  &  0.9660 / 0.1080  &  0.9120 / 0.0000 \\
TP / FP $\wvec_{2}$ & -   & - &        0.9360 / 0.4800 &   0.9000 /
0.1640 &   0.9660 / 0.6140  &  0.9380 / 0.1380 \\ 
Signal Recovery $\hat{\tenX}$ &         1.0090  &  1.0025  &  0.9997 &
0.1252  &  0.1238  &  0.1234 \\
\hline
\end{tabular}}
\label{tab_sim}
\caption{\em \footnotesize True and false positive rates (TP and FP) and signal
  recovery measured in mean 
squared error for the four simulation scenarios described in Section
  ~\ref{section_sims} averaged over 50 replicates.  The Sparse CP-TPA
  method performs similarly to Sparse HOSVD and Sparse HOOI, with the
  later having a slight advantage in some scenarios.  The Sparse
  CP-ALS method performs poorly both in terms of feature selection and
  signal recovery.} 
\end{table*}

Next in Table ~\ref{tab_sim}, we compare the performance of our
methods in terms of feature selection and low rank signal recovery,
measured in mean squared error, at
one value of the regularization parameter, $\lambda$.  To be
consistent, the BIC method was used to select $\lambda$ for all
methods. Again, the Sparse HOOI is the best performing method with the
Sparse CP-TPA method performing similarly in all but Scenario 3.
Simulation results under different percentages of sparsity and
different signal levels presented in the Supplemental Materials exhibit
similar behavior.  

Finally, we compare the time until convergence for each of our methods
for Sparse HOPCA in Table \ref{tab_time}.  Note that while convergence
cannot be mathematically guaranteed for the methods discussed in Section
\ref{section_heuristic}, in practice effective use of warm starts and
constraints on the change in factors permitted in each iteration often
yield convergent algorithms.  Timings were
carried out on a Intel Xeon X5680 3.33Ghz processor and methods were
coded as single-thread processes 
run in Matlab utilizing the Tensor Toolbox \citep{tensor_toolbox}.
From these timing results, we see that
the Sparse CP methods are considerably faster than those based on the
Tucker decomposition.  While the Sparse CP-ALS method is fastest, it
also has the worst performance.  Overall, in addition to having nice
mathematical properties, the Sparse CP-TPA method offers a good
compromise between fast computation and strong performance results in
terms of feature selection and signal recovery.

\begin{table}[!!t]
\begin{center}
\scalebox{.9}{
\begin{tabular}{|l|r|r|r|r|}
\hline
& $100 \times 100 \times 100$ & $1000 \times 20 \times 20$ & $250
\times 250 \times 250$ & $5000 \times 50 \times 50$ \\
\hline
Sparse CP-ALS &     0.075  &  0.030  &  1.665 &   3.240 \\
Sparse CP-TPA &   0.090 &   0.030 &   1.800 &   5.335 \\
Sparse HOSVD  &  0.720 &   0.330 &  13.355 &  24.300 \\
Sparse HOOI  &  0.835  &  0.430 &  14.230 &  26.950 \\
\hline
\end{tabular}}\end{center}
\caption{ \em \footnotesize Median time in seconds over ten replicates
  for convergence of 
  each algorithm for a fixed value of the regularization parameter.  A
  rank-one model in which each factor is 50\% sparse is employed.  The
  Sparse CP 
  methods are markedly faster than the Sparse Tucker methods. }
\label{tab_time}
\end{table}

\subsection{AGEMAP Microarray Example}

We analyze the high-dimensional AGEMAP microarray data set (publicly
available from {\tt
  http://www.grc.nia.nih.gov/branches/rrb/dna/agemap\_data.htm }) with
Sparse 
HOPCA.  This data consists of gene expression measurements for 8,932
genes measured for 16 tissue types on 40 mice of ages 1, 6, 16, or 24
months \citep{zahn_agemap_2007}.  As measurements for several mice are
missing for 
various tissues, we eliminate any tensor slices that are entirely
missing, yielding a data array of dimension $8932 \times 16 \times
22$.  Scientists seek to discover
relationships between tissue types of aging mice and
the subset of genomic patterns that contribute to these
relationships. These patterns in each tensor mode cannot be found by
simply applying PCA or Sparse PCA to the fattened tensor.

\begin{figure}
\includegraphics[width=6.5in,clip=true,trim=.75in 0in .75in 0in]{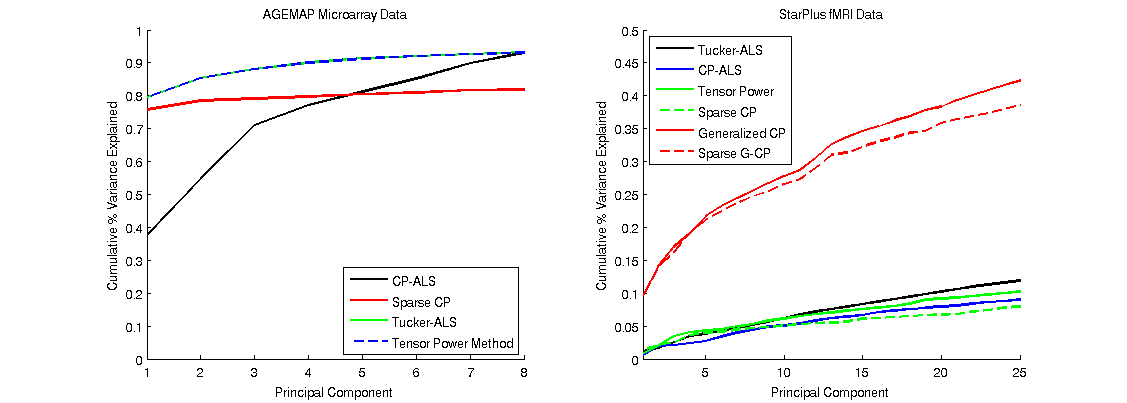}
\caption{\em \footnotesize Cumulative proportion of variance explained
by third order PCs on the AGEMAP microarray data (left) and the
StarPlus fMRI data (right).}
\label{fig_varex}
\end{figure}

The CP-TPA, CP-ALS, and Tucker decompositions as well as Sparse CP-TPA were
applied to this data to reduce the dimension and 
understand patterns among tissues and genes.  In the left panel of Figure
\ref{fig_varex}, the cumulative proportion of variance explained by
the first eight components is given.  Notice that the CP-ALS, CP-TPA, and
Tucker decompositions explain roughly the same proportion of variance
with eight components, but the CP-TPA and Tucker decomposition explain
much more initial variance in the first several PCs.  As often
scientists are only interested in 
the first couple principal components, this illustrates an important
advantage of our CP-TPA and Sparse CP-TPA approach as compared to CP-ALS
in the analysis of real data.

In Figure \ref{fig_agemap}, we explore patterns found in the
AGEMAP data via Sparse HOPCA.  The results were computed
using the Sparse CP-TPA method placing a penalty on the gene mode with the
BIC used to select $\lambda$.  In the top panel, we show
scatterplots of the first six 
PCs for the tissue mode.  We see many clusters of
tissue types for the various pairs of PCs.  For
example, adrenal, gonads and bones often cluster together.  As only a
subset of genes are selected by each of the PCs, we
can analyze the genetic patterns further for each tissue type.  Gonads
has a higher value for the second principal component, for example, so
we display a cluster heatmap of the 1,439 genes selected in PC2 for
this tissue type in the lower 
left panel of Figure \ref{fig_agemap}.    We see that the
genes selected 
by this sparse PC perfectly separate the male and female mice.  As
liver has a lower PC value for the third tissue component, we display the
cluster heatmap for the 514 genes selected by this component for liver
in the lower right panel of Figure 3.  Again, we see that this
component clusters 
the mice well according to their ages.  Further plots and analyses of this
type reveal subsets of important genes associated with various
tissues and mice ages or gender.  This type of multi-mode analysis is an
important advantage of applying Sparse HOPCA as opposed to Sparse PCA
on a flattened tensor.

\begin{figure}[!!th]
\begin{center}
\includegraphics[width=8in,clip=true,trim=2.2in 0in 2.5in 0in]{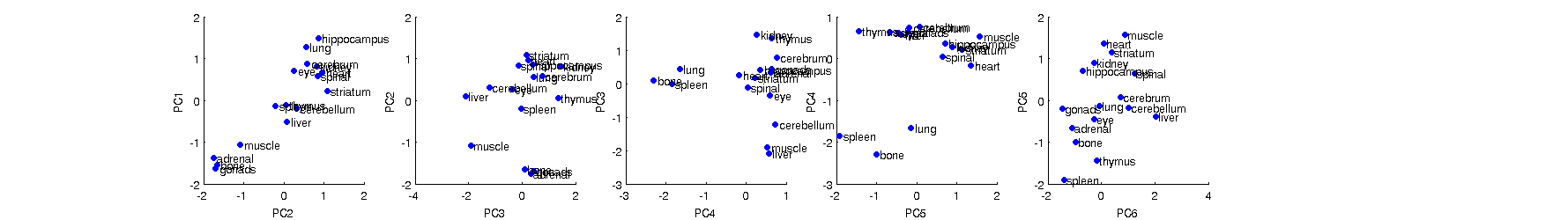} \\
\includegraphics[width=3.5in,clip=true,trim=.4in 0in .4in 0in]{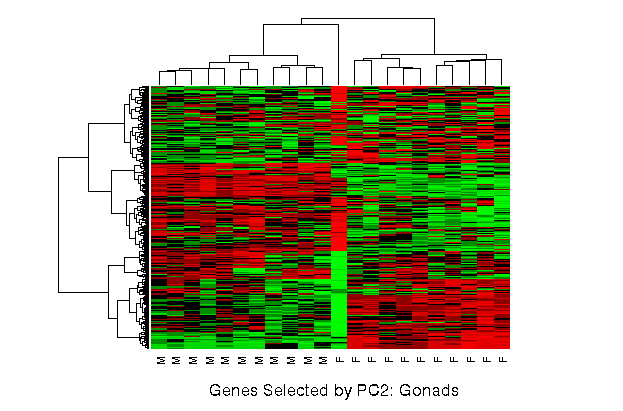}\includegraphics[width=3.5in,clip=true,trim=.4in 0in .4in 0in]{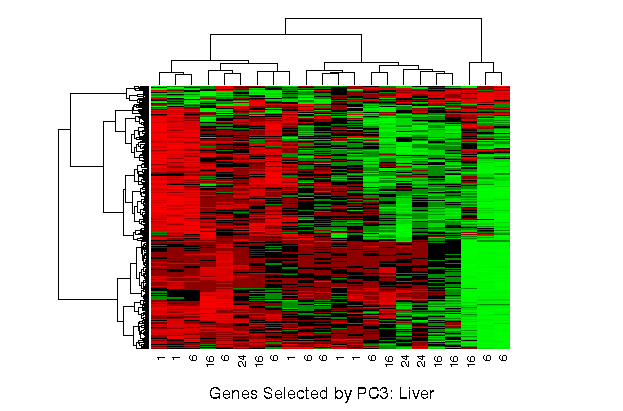}
\end{center}
\caption{ \em \footnotesize Analysis of AGEMAP Microarray data via
  Sparse Higher-Order PCA.  (Top Panel) Scatterplots of the first
  eight principal components for the 16 tissue types.  (Lower Left) A
  cluster heatmap of the genes selected by sparse PC
  3 for the tissue Gonads labeled by gender.  (Lower Right) A
  cluster heatmap of the genes selected by sparse PC
  8 for the tissue Heart labeled by age in months. }
\label{fig_agemap}
\end{figure}

\subsection{StarPlus fMRI Example}

As functional MRIs are a common source of high-dimensional tensor
data, we apply our methods to understand patterns for 
subject 04847 of the StarPlus fMRI experiment (publicly available from {\tt
  http://www.cs.cmu.edu/afs/cs.cmu.edu/project/theo-81/www/})
\citep{mitchell_starplus_2004}.  This data set consists of 4,698
voxels or spatial locations in the brain sampled on a $64 \times 64
\times 8$ grid, measured over 54 - 55 time
points for each of 40 tasks.  In each task, the subject was shown an
image and read a sentence.  The sentence either explained the image,
for example an image of a star with a plus above paired with the
sentence, ``The plus is above the star.'', or in which the sentence
negated the image.  We analyze data for each of the 36 tasks lasting
for 55 time points, yielding a tensor array of dimension $4,698 \times
55 \times 36$.

We apply HOPCA, Sparse HOPCA, Generalized CP, and Sparse Generalized CP-TPA
methods to understand the spatial, task and temporal patterns in the
fMRI data \citep{mrup_neuroimage_2008}.  The Generalized HOPCA methods
were applied with $\Q^{(1)}$ 
a graph Laplacian of 
the nearest neighbor graph connecting the voxels, $\Q^{(2)}$ a kernel
smoother over the 55 time points and $\Q^{(3)} = \mathbf{I}$.  These
quadratic operators were selected such that the variance explained by
the first GPC was maximal as described in \citet{allen_gmd_2011}.
Sparsity was incorporated into the spatial component to select
relevant regions of interest.  In the right panel of Figure
\ref{fig_varex},  we present the cumulative proportion of variance
explained by each of the methods.  Generalized HOPCA methods
strikingly explain a much great proportion of variance as they
incorporate known spatio-temporal structure into the tensor
factorization.  In Figure \ref{fig_fmri}, we compare the first two
HOPCs of the Tucker decomposition to those of the Sparse Generalized CP-TPA
method.  The latter yield interpretable temporal patterns and discrete
regions of interest related to the tasks.  The temporal patterns of
the Tucker decomposition appear to be confounded by sinusoidal noise
with noisy, uninterpretable spatial components.

\begin{figure}
\includegraphics[width=6in]{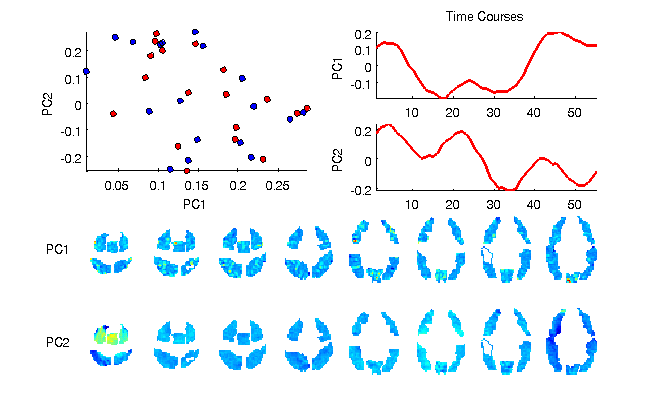} 
\includegraphics[width=6in]{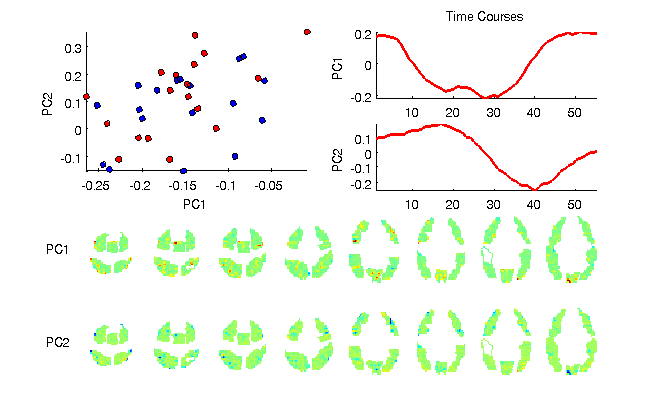}
\caption{\em \footnotesize Depictions of the first two task PCs, time
  course PCs 
  and spatial PCs for the Tucker decomposition (top panel) and Sparse
  Generalized CP-TPA (bottom panel) for the StarPlus fMRI data.  In
  the scatterplots, red denotes the sentence negating the image while
  blue denotes sentence and image agreement.}
\label{fig_fmri}
\end{figure}

\section{Discussion}
\label{section_dis}

We have developed methodology for regularizing HOPCA in the context of
high-dimensional tensor data.  Beginning 
with algorithmic approaches, we presented Sparse CP-ALS, Sparse HOSVD,
and Sparse HOOI methods based on popular algorithms for the CP and
Tucker decompositions.  While these methods are intuitive, they fail to
solve a coherent optimization problem, and are thus
mathematically and computationally less desirable.  Next, we develop a
greedy framework for computing HOPCs based on deflation and show that
a simple multi-way concave relaxation of this optimization problem
leads to an approach we term Sparse CP-TPA.  As this method converges
to a local solution of a well-defined optimization problem, it enjoys
many mathematical and computational advantages.  Finally, we show how
the deflation approach to HOPCA can be extended to work with general
penalties and constraints, and with structured or functional tensor
data.  A major strength of our approach is its flexibility to model
different types of high-dimensional tensor data, a quality illustrated
on our examples to microarray and fMRI data.

There are several items related to our work to study further.  First,
we note that while we have presented all of our 
methods with three-mode tensors for notational convenience, extensions
to tensors with an arbitrary number of modes is trivial.  The
extensions of our Sparse HOPCA frameworks to encompass general
penalties and constraints, structured tensors, and multi-way
functional data outlined in Section \ref{section_ext} deserve further
investigation and development.  Simulation studies and comparisons to
existing approaches are needed to fully evaluate the utility of these
methods.

Two important items related to our regularized HOPCA frameworks are
beyond the scope of this paper and require further investigation.
Determining how many HOPCs to extract for a given data set is an
important problem.  For PCA on matrix data, there are several
data-driven methods
available \citep{buja_1992, owen_2009_cv}, but extending these
approaches to the tensor domain 
is a non-trivial task.  Our result on the amount of variance explained
by each of the HOPCs or Sparse HOPCs, Theorem \ref{thm_var_ex}, gives
users of our methodology a way to gage the dimension 
reduction achieved, but further work needs to be done to determine
the optimal number of factors.  Secondly, results for every
regularized HOPCA method, and in fact all methods for HOPCA,
depend heavily on initial algorithmic starting values
\citep{kroonenberg_2008}.  This occurs as 
all HOPCA and regularized HOPCA methods find at best a local optimum.
With Sparse PCA 
methods on matrix data, the same problem occurs but is less critical
as one can initialize algorithms to the SVD which is a  global
solution.  With tensor factorizations, further work is required to
determine the best initializations for our regularized HOPCA
algorithms.

As the statistical attention given to high-dimensional tensor data has
been limited, topics for further study abound.  Recently, many have
made progress on understanding the asymptotic properties of PCA and
Sparse PCA by using random matrix theory \citep{johnstone_jasa_2009,
  jung_pca_2009, amini_2009}.  As yet, there is no work on 
developing consistency results for HOPCA which would serve as a
precursor to studying consistency for the methods presented in this
paper.  Additionally, our framework for
regularized HOPCA is a foundation upon which to study other
multivariate analysis techniques for tensor data.  These include
methods for clustering, canonical correlations analysis, partial
least squares, discriminant analysis, and multi-dimensional scaling.
Variable selection in the context of regression and classification for
high-dimensional matrix data has been a topic of great interest
recently.  One can imagine that variable selection for
high-dimensional tensors may be more critical in these supervised
prediction methods.

Finally, there are a plethora of possible applications of our
methodology.  Technologies in biomedical imaging are producing massive
multi-way data sets that are a challenge to understand and analyze.
Examples of these are especially common in neuroimaging, radiology,
and chemometrics.  Other imaging data sets such as from hyper-spectral
cameras and remote sensing often contain three or more measured
dimensions.  Additional examples of high-dimensional tensor data come from
environmental and climate studies, bibliometrics, and network
reliability.  Overall, our methods for regularized HOPCA are 
foundational tools for understanding 
high-dimensional tensors that will enjoy wide-ranging applicability.







\section{Acknowledgments}

This paper is based, in part, on a previous conference paper
\citet{allen_shopca_2012}.  The author would like to thank Eric Chi
for helpful 
discussions on tensors.  Supplemental materials contain proofs of all
theoretical results, algorithm outlines for all methods, and
additional simulation results.  Software based
on the {\tt Matlab Tensor Toolbox} \citep{tensor_toolbox} will be made
available from {\tt http://www.stat.rice.edu/$~$gallen/software.html}.

\appendix

\section{Algorithms for HOPCA \& Sparse HOPCA}

{\footnotesize
\begin{algorithm}
\caption{Sparse Higher-Order SVD (Sparse HOSVD)}
\label{shosvd_alg}
\begin{enumerate}
\item {\small $\U \leftarrow$ First $K_{1}$ sparse principal
  components of $\X_{(1)}$}. 
\item {\small $\V \leftarrow$ First $K_{2}$ sparse principal components of $\X_{(2)}$}.
\item {\small $\W \leftarrow$ First $K_{3}$ sparse principal components of $\X_{(3)}$}.
\item {\small $\tenD \leftarrow \tenX \times_{1} \U \times_{2} \V \times_{3} \W$}.
\end{enumerate}
\end{algorithm}

\begin{algorithm}
\caption{Sparse Higher-Order Orthogonal Iteration (Sparse HOOI), or
  Sparse Tucker-ALS} 
\label{shooi_alg}
\begin{enumerate}
\item Initialize $\U$, $\V$, and $\W$ using the
  Sparse HOSVD algorithm.
\item Repeat until convergence or maximum number of iterations reached:
\begin{enumerate}
\item {\small $\U \leftarrow$ First $K_{1}$ sparse principal
  components of $(\tenX \times_{2} \V \times_{3} \W)_{(1)}$}. 
\item {\small $\V \leftarrow$ First $K_{2}$ sparse principal
  components of $(\tenX \times_{1} \U \times_{3} \W)_{(2)}$}.
\item {\small $\W \leftarrow$ First $K_{3}$ sparse principal
  components of $(\tenX \times_{1} \U \times_{2} \V)_{(3)}$}.
\end{enumerate}
\item {\small $\tenD \leftarrow \tenX \times_{1} \U \times_{2} \V
  \times_{3} \W$}
\end{enumerate}
\end{algorithm}

\begin{algorithm}
\caption{Sparse CP-ALS Algorithm}
\begin{enumerate}
\item Initialize $\U^{(0)}$, $\V^{(0)}$ and $\W^{(0)}$.
\item Repeat until convergence or maximum number of iterations
  reached:
\begin{enumerate}
\item Estimate $\U^{(t+1)}$:
\begin{enumerate}
\item $\hat{\U} \leftarrow \argmin_{\U} \left\{ \frac{1}{2} || \X_{(1)}
  - \U ( \V^{(t)} \odot \W^{(t)} )^{T} ||_{F}^{2}  + \lambda_{\U} ||
  \U ||_{1} \right\}.$
\item $\hat{d}_{k} \leftarrow || \hat{\uvec}_{k} ||.$
\item $\uvec_{k}^{(t+1)} \leftarrow \hat{\uvec}_{k} / \hat{d}_{k}$.
\end{enumerate}
\item Estimate $\V^{(t+1)}$:
\begin{enumerate}
\item $\hat{\V} \leftarrow \argmin_{\V} \left\{ \frac{1}{2} || \X_{(2)}
  - \V ( \U^{(t+1)} \odot \W^{(t)} )^{T} ||_{F}^{2}  + \lambda_{\V} ||
  \V ||_{1} \right\}.$
\item $\hat{d}_{k} \leftarrow || \hat{\vvec}_{k} ||.$
\item $\vvec_{k}^{(t+1)} \leftarrow \hat{\vvec}_{k} / \hat{d}_{k}$.
\end{enumerate}
\item Estimate $\W^{(t+1)}$:
\begin{enumerate}
\item $\hat{\W} \leftarrow \argmin_{\W} \left\{ \frac{1}{2} || \X_{(3)}
  - \W ( \U^{(t+1)} \odot \V^{(t+1)} )^{T} ||_{F}^{2}  + \lambda_{\W} ||
  \W ||_{1} \right\}.$
\item $\hat{d}_{k} \leftarrow || \hat{\wvec}_{k} ||.$
\item $\wvec_{k}^{(t+1)} \leftarrow \hat{\wvec}_{k} / \hat{d}_{k}$.
\end{enumerate}
\end{enumerate}
\end{enumerate}
\end{algorithm}

\begin{algorithm}
\caption{Tensor Power Algorithm}
\begin{enumerate}
\item Initialize $\hat{\tenX} = \tenX$.
\item For $k = 1 \ldots K$
\begin{enumerate}
\item Repeat until converge:
\begin{enumerate}
\item $\uvec_{k} \leftarrow \hat{\tenX} \times_{2} \vvec_{k} \times_{3}
  \wvec_{k} \  / \ || \hat{\tenX} \times_{2} \vvec_{k} \times_{3}
  \wvec_{k} ||_{2}$.
\item $\vvec_{k} \leftarrow \hat{\tenX} \times_{1} \uvec_{k} \times_{3}
  \wvec_{k} \  / \ || \hat{\tenX} \times_{1} \uvec_{k} \times_{3}
  \wvec_{k} ||_{2}$.
\item $\wvec_{k} \leftarrow \hat{\tenX} \times_{1} \uvec_{k} \times_{2}
  \vvec_{k} \  / \ || \hat{\tenX} \times_{1} \uvec_{k} \times_{2}
  \vvec_{k} ||_{2}$.
\end{enumerate}
\item $d_{k} \leftarrow \hat{\tenX} \times_{1} \uvec_{k} \times_{2}
  \vvec_{k} \times_{3} \wvec_{k}$.
\item $\hat{\tenX} \leftarrow \hat{\tenX} - d_{k} \uvec_{k} \circ
  \vvec_{k} \circ \wvec_{k}$.  
\end{enumerate}
\end{enumerate}
\label{tensor_pow}
\end{algorithm}

\begin{algorithm}[!!h]
\caption{Sparse CP-TPA}
\label{tensor_pow}
{\small
\begin{enumerate}
\item Initialize $\hat{\tenX} = \tenX$.
\item For $k = 1 \ldots K$
\begin{enumerate}
\item Repeat until converge:
\begin{enumerate}
\item $\hat{\uvec}_{k} = S \left( \hat{\tenX} \times_{2} \vvec_{k} \times_{3}
  \wvec_{k}, \lamu \right)$.  \hspace{3mm}
$\uvec_{k} \leftarrow  \begin{cases}  \hat{\uvec}_{k} /  ||
  \hat{\uvec}_{k} ||_{2} & || \hat{\uvec}_{k} ||_{2} > 0 \\ 0 &
  \textrm{otherwise}. 
\end{cases}$
\item  $\hat{\vvec}_{k} = S \left( \hat{\tenX} \times_{1} \uvec_{k} \times_{3}
  \wvec_{k}, \lamv \right)$.  \hspace{3mm}
$\vvec_{k} \leftarrow \begin{cases}  \hat{\vvec}_{k} / ||
  \hat{\vvec}_{k} ||_{2} & || \hat{\vvec}_{k} ||_{2} > 0 \\ 0 &
  \textrm{otherwise}.   \end{cases}$
\item $\hat{\wvec}_{k} = S \left( \hat{\tenX} \times_{1} \uvec_{k} \times_{2}
  \vvec_{k}, \lamw \right)$.  \hspace{3mm}
$\wvec_{k} \leftarrow  \begin{cases} \hat{\wvec}_{k} /  ||
  \hat{\wvec}_{k} ||_{2} & || \hat{\wvec}_{k} ||_{2} > 0 \\ 0 &
  \textrm{otherwise}. \end{cases}$
\end{enumerate}
\item $d_{k} \leftarrow \hat{\tenX} \times_{1} \uvec_{k} \times_{2}
  \vvec_{k} \times_{3} \wvec_{k}$.
\item $\hat{\tenX} \leftarrow \hat{\tenX} - d_{k} \uvec_{k} \circ
  \vvec_{k} \circ \wvec_{k}$.  
\end{enumerate}
\end{enumerate}
}
\end{algorithm}
}

\section{Additional Simulation Results}

\begin{sidewaysfigure}
\begin{center}
\includegraphics[width=4.5in]{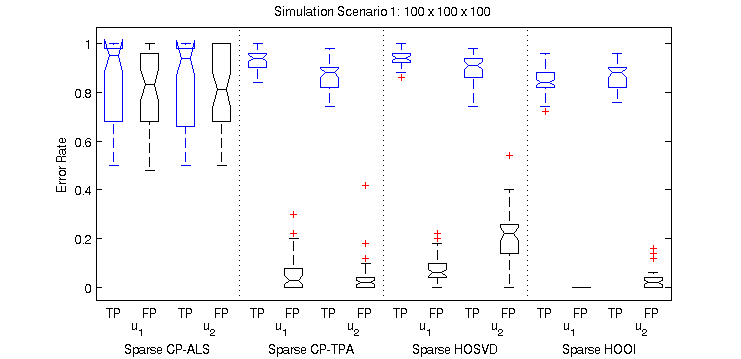}\includegraphics[width=4.5in]{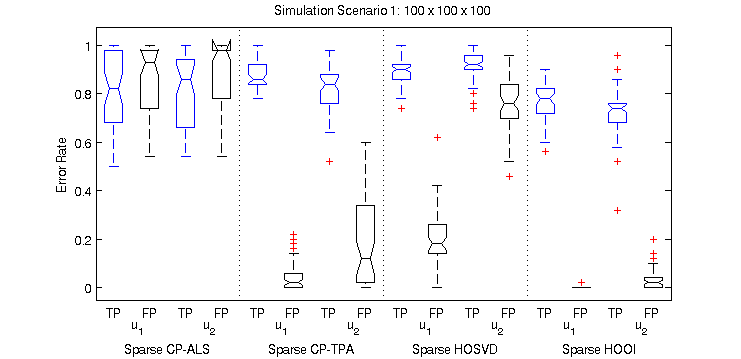} \\
\vspace{.2in}
\includegraphics[width=4.5in]{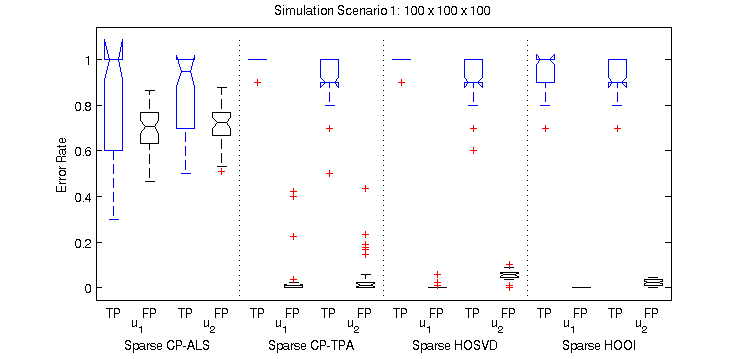}\includegraphics[width=4.5in]{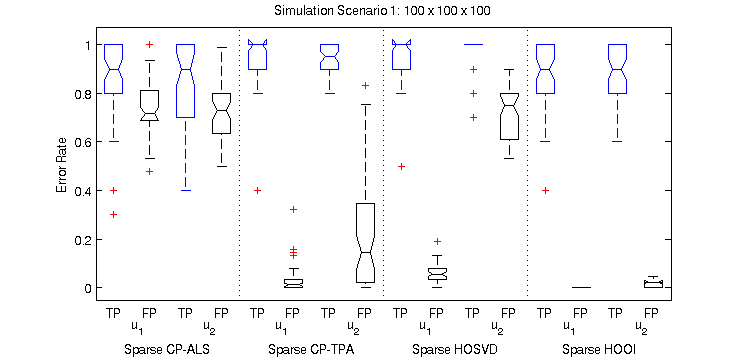}
\end{center}
\caption{ \em \footnotesize Boxplots of true and false positives for
  simulation scenario 1, $K=2$, for factors with 50\% sparsity (upper panels)
  or 90\% sparsity (lower panels) and with high signal (left panels),
  $D = [200 \ \ 100]^{T}$ or lower signal (right panels), $D = [100 \
  \ 50]^{T}$.} 
\end{sidewaysfigure}

\begin{sidewaysfigure}
\begin{center}
\includegraphics[width=4.5in]{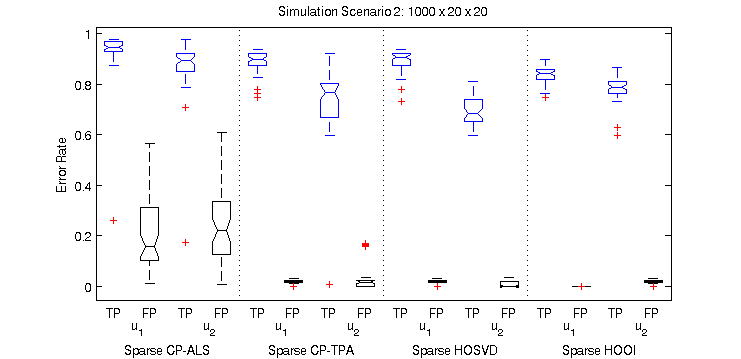}\includegraphics[width=4.5in]{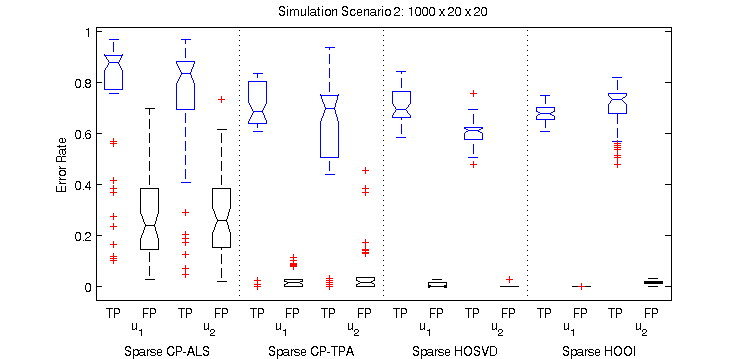} \\
\vspace{.2in}
\includegraphics[width=4.5in]{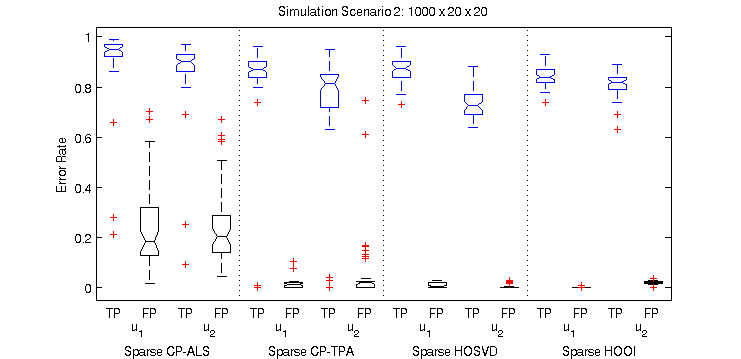}\includegraphics[width=4.5in]{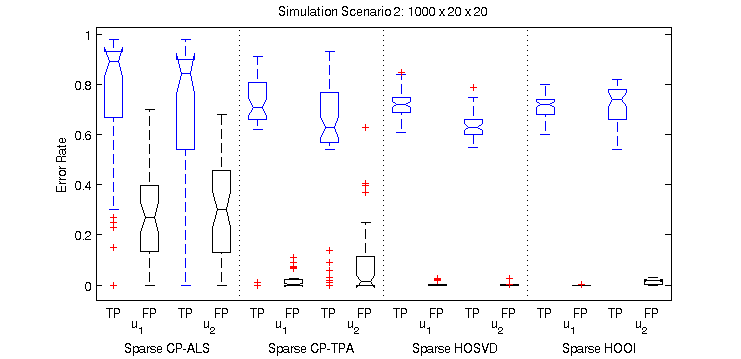}
\end{center}
\caption{ \em \footnotesize Boxplots of true and false positives for
  simulation scenario 2, $K=2$, for factors with 50\% sparsity (upper panels)
  or 90\% sparsity (lower panels) and with high signal (left panels),
  $D = [200 \ \ 100]^{T}$ or lower signal (right panels), $D = [100 \
  \ 50]^{T}$.} 
\end{sidewaysfigure}

\section{Proofs}

\begin{proof}[Proof of Theorem \ref{thm_var_ex}.]
The proof is an extension of a result in \citet{shen_spca_2008}.  Recall
that the cumulative proportion of variance explained by the first $k$
traditional principal components can be written as the ratio of the
squared Frobenius norm of the data projected onto the first $k$ left
and right singular vectors to the squared Frobenius norm of the data
\citep{jolliffe_2002}.  Notice that $\Pmat^{(U)}_{k}$,
$\Pmat^{(V)}_{k}$ and $\Pmat^{(W)}_{k}$ are projection
matrices with exactly $k$ eigenvalues equal to one.  Therefore,
the result is an extension of the definition of cumulative proportion
of variance explained to the tensor framework.  Namely, the
numerator is simply 
the projection of the tensor onto the subspace spanned by the first $k$
HOPCA factors.  
\end{proof}

\begin{proof}[Proof of Proposition \ref{prop_cp_als}]
We will show that the updates for $\U$ in the Sparse CP-ALS algorithm
are not equivalent to the optimal point for $\U$ from
\eqref{not_cp_als}.  Arguments for $\V$ and $\W$ are analogous.
First, note that the problem \eqref{not_cp_als} with respect to $\U$
can be rewritten using the 
Khatri-Rao product as $\minimize_{\U} \ \frac{1}{2} || \X_{(1)} - \U (
\V \odot \W )^{T} ||_{F}^{2} + \lamu || \U ||_{1} \ \textrm{subject
  to} \ \uvec_{k}^{T} \uvec_{k} \leq 1 \ \forall \ k=1, \ldots K$.  To
simplify notation, we will write $\Hmat = \V \odot \W$ and $\Y =
\X_{(1)}$.  As this problem is convex in $\U$, the KKT conditions are
necessary and sufficient for optimality.  These conditions include
the sub-gradient equation:
\begin{align}
\label{subgrad_not_cp_als}
\Y \Hmat^{T} - \U^{*} \Hmat \Hmat^{T} + \lamu \Gamma( \U^{*} ) - 2
\Psi^{*} \U^{*} = 0,
\end{align}
where $\Gamma()$ is the sub-gradient of the $\ell_{1}$-norm and
$\Psi^{*}$ is the Lagrange multiplier, namely a diagonal matrix that
from complimentary slackness is strictly non-zero, $\Psi^{*} \succ 0$
if and only if the columns of $\U^{*}$ have norm one.  

Now consider the sub-gradient equations implied by the Sparse CP-ALS
updates for $\U$ before scaling the solution $\hat{\U}$ to have column
norm one:
\begin{align*}
\Y \Hmat^{T} - \hat{\U} \Hmat \Hmat^{T} + \lamu \Gamma( \hat{\U} ) = 0.
\end{align*}
We can represent scaling the columns of $\hat{\U}$ as right
multiplying by a diagonal matrix: $\hat{\U} \D$.  Since $\Gamma()$ is
order one, then $\Gamma( \hat{\U} \D ) = \Gamma( \hat{\U} )$ and the
above sub-gradient equation is equal to:
\begin{align*}
\Y \Hmat^{T} - (\hat{\U} \D) \D^{-1} \Hmat \Hmat^{T} + \lamu \Gamma(
\hat{\U} \D ) &= 0,
\Y \Hmat^{T} - (\tilde{\U}) \D^{-1} \Hmat \Hmat^{T} + \lamu \Gamma(
\tilde{\U} ) &= 0,
\end{align*}
where $\tilde{\U} = \hat{\U} \D$.  Now, for the Sparse CP-ALS update
to solve \eqref{not_cp_als}, the above sub-gradient equation must be
equivalent to \eqref{subgrad_not_cp_als} for some diagonal $\Psi^{*}
\succ 0$ and some diagonal $\D \succ 0$.  This means that $\U^{*} (
\Hmat \Hmat^{T} + 2 \Psi^{*}) = \tilde{\U} \D^{-1} \Hmat \Hmat^{T}$.
Assuming that $\tilde{\U} = \U^{*}$ and solving for $\Psi^{*}$, we
have that $\Psi^{*} = \Hmat \Hmat^{T} ( \D^{-1} - \mathbf{I} ) / 2$.
Thus, $\Psi^{*}$ is diagonal if and only if $\Hmat \Hmat^{T}$ is
diagonal.  This is a contradiction.  Therefore, the updates from the
Sparse CP-ALS algorithm do not solve \eqref{not_cp_als}.   
\end{proof}

\begin{proof}[Proof of Proposition \ref{prop_tpa}]
Consider optimizing \eqref{cp_max} with respect to $\uvec$.  The
Lagrangian is given by $L(\uvec, \gamma) = ( \tenX \times_{2} \vvec
\times_{3} \wvec ) \times_{1} \uvec - \gamma ( \uvec^{T} \uvec - 1)$.
The KKT conditions imply that $\uvec^{*} = \frac{ \tenX \times_{2}
  \vvec \times_{3} \wvec}{2 \gamma^{*} }$ and $\gamma^{*}$ is such
that $( \uvec^{*})^{T} \uvec^{*} = 1$.  Putting these together we have
the desired result for $\uvec$.  The arguments for $\vvec$ and $\wvec$
are analogous.  
\end{proof}

\begin{proof}[Proof of Theorem \ref{thm_sparse_cp}]
The proof follows from an extension of results in
\citet{witten_pmd_2009} and \citet{allen_gmd_2011}.  In short, consider
optimizing \eqref{sparse_cp_prob} with respect to $\uvec$.  The KKT
conditions imply that $\tenX \times_{2} \vvec \times_{3} \wvec -
\rho_{\uvec} \mathbf{\Gamma}(\uvec^{*} ) - 2 \gamma^{*} \uvec^{*} = 0$
and $\gamma^{*} ( (\uvec^{*})^{T} \uvec^{*} - 1) = 0$ where
$\mathbf{\Gamma}(\uvec)$ is the subgradient of $|| \uvec||_{1}$, and
$\gamma$ is a Lagrange multiplier.  Consider $\hat{\uvec} = S( \tenX
\times_{2} \vvec \times_{3} \wvec, \rho_{\uvec})$.  Then, taking
$\uvec^{*} = \hat{\uvec} / || \hat{\uvec} ||_{2}$ and $\gamma^{*} = ||
\hat{\uvec} ||_{2} / 2$ simultaneously satisfies the KKT conditions.
Since the problem is convex in $\uvec$, the conditions are necessary
and sufficient; hence, the pair $( \uvec^{*}, \gamma^{*})$ are the
optimal points.  It is easy to verify that the pair $(0,0)$ also
satisfy the KKT conditions and are optimal points.
\end{proof}

\begin{proof}[Proof of Theorem \ref{thm_cp_gen_pen}]
The proof follows from an extension of results in
\citet{allen_gmd_2011}.  Let us briefly consider the case for $\uvec$
and the arguments for $\vvec$ and $\wvec$ are analogous.  We will show
that the optimal points, $(\uvec^{*}, \vvec^{*})$, implied by
\eqref{cp_prob_reg} are equivalent 
to the said solution.  The sub-gradient equation for
\eqref{cp_prob_reg} is: $\tenX \times_{2} \vvec \times_{3} \wvec - 2
\gamma^{*} \uvec^{*} - \lamu \nabla P_{\uvec}( \uvec^{*} ) = 0$, where $\nabla
P_{\uvec}()$ is the sub-gradient of $P_{\uvec}()$.  Now, consider the
subgradient equation of $\minimize_{\uvec} \ \frac{1}{2} || \tenX
\times_{2} \vvec \times_{3} \wvec - \uvec ||_{2}^{2} + \lamu
P_{\uvec}( \uvec )$ and for some $c>0$:
\begin{align*}
0 = \tenX \times_{2} \times_{3} \wvec - \hat{\uvec} - \lamu \nabla
P_{\uvec} ( \uvec ) = \tenX \times_{2} \times_{3} \wvec - \frac{1}{c}(
c \hat{\uvec}) -  \lamu \nabla P_{\uvec} ( \uvec ) =  \tenX \times_{2}
\times_{3} \wvec - \tilde{\uvec} / c  - \lamu \nabla P_{\uvec} (
\tilde{\uvec} ),
\end{align*}
where $\tilde{\uvec} = c \hat{\uvec}$.  Since $P_{\uvec}()$ is order
one, $\nabla P_{\uvec} ( x) = \nabla P_{\uvec}( c x) \ \forall c >
0$.  Taking $c = 1/ || \hat{\uvec} ||_{2}$ and letting $\gamma^{*} =
\frac{1}{2c} = || \hat{\uvec} ||_{2} / 2$, we see that for the pair $(
\uvec^{*} = \hat{\uvec} / || \hat{\uvec} ||_{2}, \gamma^{*} = ||
\hat{\uvec} ||_{2} / 2)$ the subgradient equation of the penalized
regression problem is equivalent to that of \eqref{cp_prob_reg}.
Following from complimentary slackness, $\gamma^{*} = 0$ if and only
if $\hat{\uvec} \equiv 0$, and hence the pair $(0,0)$ also satisfy the
KKT conditions of \eqref{cp_prob_reg}.  We have thus proven the
desired result.
\end{proof}

\begin{proof}[Proof of Proposition \ref{prop_ghopca}]
The proof for part one follows from an extension of Proposition
\ref{prop_tpa} to the generalized eigenvalue case as shown in
\citep{allen_gmd_2011}.  In brief, the gradient equation of
\eqref{gcp} is $\tenX \times_{1} \Q^{(1)} \times_{2} \Q^{(2)} \vvec
\times_{3} \Q^{(3)} \wvec - 2 \gamma \Q^{(1)} \uvec = 0$ meaning that
$\uvec^{*} = \tenX  \times_{2} \Q^{(2)} \vvec
\times_{3} \Q^{(3)} \wvec / 2 \gamma^{*}$ where $\gamma^{*}$ is such
that $\uvec^{T} \Q^{(1)} \uvec^{*} = 1$.  Putting these together we
have the desired result.  

The argument for part two follows from an extension of Theorem
\ref{thm_cp_gen_pen} and \citet{allen_gmd_2011}.  For completeness, we
show that the subgradient of \eqref{sparse_gcp} with respect to
$\uvec$, \\
$\tenX \times_{1} \Q^{(1)} \times_{2}
\Q^{(2)} \vvec \times_{3} \Q^{(3)} \wvec - 2 \gamma \Q^{(1)} \uvec -
\lamu \Gamma( \uvec) = 0$ where $\Gamma( \cdot )$ is the sub-gradient of $||
\cdot ||_{1}$ is equivalent to the Sparse Generalized CP update given
in Proposition \ref{prop_ghopca}.  Let
$\y = \tenX \times_{2} \Q^{(2)} \vvec \times_{3} \Q^{(3)}
\wvec$ and let $\hat{\uvec}$ be the argument minimizing $\frac{1}{2}
|| \y - \uvec ||_{\Q^{(1)}}^{2} + \lamu || \uvec ||_{1}$.  Then, for
some $c > 0$, $0 = \Q^{(1)} \y - \Q^{(1)} \hat{\uvec} - \lamu \Gamma(
\hat{\uvec} ) = 
\Q^{(1)} \y - \Q^{(1)} \tilde{\uvec} / c - \lamu \Gamma( \tilde{\uvec} )$,
where $\tilde{\uvec} = c \hat{\uvec}$, since $\Gamma ( \hat{\uvec}
) = \Gamma( c \hat{\uvec} ) \ \forall c > 0$.  Taking $\gamma^{*} =
1/2c = || \hat{\uvec} ||_{\Q^{(1)}} / 2$ we see that the for both
pairs $( \uvec^{*} = \hat{\uvec} / || \hat{\uvec} ||_{\Q}^{(1)},
\gamma^{*} = || \hat{\uvec} ||_{\Q^{(1)}} / 2)$ and $(\uvec^{*} = 0,
\gamma^{*} = 0)$ satisfy the KKT conditions of \eqref{sparse_gcp},
thus proving the desired result.  
\end{proof}

\begin{proof}[Proof of Proposition \ref{prop_hofpca}]
For part 1, consider the gradient equation of \eqref{tensor_fpca} with
respect to $\uvec$: $\frac{\partial}{\partial \uvec} = -2 \tenX
\times_{2} \vvec \times_{3} \wvec + 2 \uvec \vvec^{T} \vvec \wvec^{T}
\wvec - 2 \uvec || \vvec ||_{2}^{2} || \wvec ||_{2}^{2} + 2
\Smat_{\uvec} \uvec \vvec^{T} \Smat_{\vvec} \vvec \wvec^{T}
\Smat_{\wvec} \wvec = 0$.  It is then easy to see that the coordinate
update is $\uvec^{*} = \Smat_{\uvec}^{-1} \tenX \times_{2} \vvec
\times_{3} \wvec / ( \vvec^{T} \Smat_{\vvec} \vvec \wvec^{T}
\Smat_{\wvec} \wvec )$.  

For part 2, let us review the two-way matrix case proved in
\citet{huang_twfpca_2009}.  They show that two-way FPCA via
half-smoothing implies 
that $\uvec \propto \Smat_{\uvec}^{-1} \X \vvec$ and $\vvec \propto
\Smat_{\vvec}^{-1} \X^{T} \uvec$ which are the stationary points of
the two gradient equations for the penalized optimization problem.
For our tensor case, the gradient equations of \eqref{tensor_fpca}
imply $\uvec \propto \Smat_{\uvec}^{-1} \tenX \times_{2} \vvec
\times_{3} \wvec$ and so forth, but these are not satisfied by
tensor half-smoothing.  As the Tucker decomposition has a non-diagonal
core, $\uvec_{1}$ is proportional to the first column of
$\Smat_{\uvec}^{-1} ( \tenX \times_{2} \V \times_{3} \W ) \tenD$.
Thus, the stationary points implied by half-smoothing and
\eqref{tensor_fpca} are not equivalent.
\end{proof}

{\footnotesize
\singlespacing
\bibliographystyle{Chicago}
\bibliography{tensors}
}

\end{document}